\documentclass[10pt,journal,compsoc]{IEEEtran}
\usepackage{color,soul} 
\definecolor{Dred}{RGB}{192,0,0} 
\definecolor{Blue}{RGB}{0,0,255} 

\ifCLASSOPTIONcompsoc
\usepackage[nocompress]{cite}
\else
\usepackage{cite}
\fi

\usepackage{hyperref} 
\hypersetup{hidelinks}
\usepackage{graphicx}
\usepackage{float}
\ifCLASSINFOpdf
\else
\fi
\usepackage{amsmath,bm}
\usepackage{amssymb}
\usepackage{algorithm}
\usepackage{algpseudocode}

\usepackage{array}

\newtheorem{theorem}{Theorem}
\newtheorem{lemma}{Lemma}
\newtheorem{proof}{Proof}
\newtheorem{remark}{Remark}

\usepackage{enumerate}

\newcommand*{\argmin}{\mathop{\mathrm{argmin}}}

\ifCLASSOPTIONcompsoc
\usepackage[caption=false,font=footnotesize,labelfont=sf,textfont=sf]{subfig}
\else
\usepackage[caption=false,font=footnotesize]{subfig}
\fi
\usepackage{multirow}  
\usepackage{booktabs} 

\usepackage{arydshln} 

\usepackage{ragged2e} 

\begin{document}
		%
		\title{IF2Net: Innately Forgetting-Free Networks for Continual Learning}
		%
		%
		%
		%
		
		\author{Depeng~Li,
			Tianqi~Wang,
			Bingrong~Xu,
			Kenji~Kawaguchi,
			Zhigang~Zeng,~\IEEEmembership{Fellow,~IEEE},\\
			and~Ponnuthurai Nagaratnam Suganthan,~\IEEEmembership{Fellow,~IEEE}
			
			\IEEEcompsocitemizethanks{
				\IEEEcompsocthanksitem D. P. Li, T. Q. Wang, and Z. G. Zeng are with the School of Artificial Intelligence and Automation, with the Institute of Artificial Intelligence, Huazhong University of Science and Technology, and also with the Key Laboratory of Image Processing and Intelligent Control of Education Ministry of China, Wuhan 430074, China. E-mail: \{dpli, tianqiwang, zgzeng\}@hust.edu.cn.
				\IEEEcompsocthanksitem B. R. Xu is with the School of Automation, Wuhan University of Technology, Wuhan 430070, China. E-mail: bingrongxu@whut.edu.cn.
				\IEEEcompsocthanksitem K. Kawaguchi is with the School of Computing, National University of Singapore, 117417, Singapore. E-mail: kawaguch@csail.mit.edu.
				\IEEEcompsocthanksitem P. N. Suganthan is with the School of Electrical and Electronic Engineering, Nanyang Technological University, Singapore 639798, and also with the KINDI Center for Computing Research, College of Engineering, Qatar University, Doha, Qatar. E-mail: epnsugan@ntu.edu.sg.
				\protect}
			
	\thanks{Manuscript received December 25, 2022.}
	\thanks{(Corresponding author: Z. G. Zeng.)}}

%
%

\markboth{Journal of \LaTeX\ Class Files,~Vol.~XX, No.~XX, December~2022}%
{Shell \MakeLowercase{\textit{et al.}}: Bare Demo of IEEEtran.cls for Computer Society Journals}
%




\IEEEtitleabstractindextext{%
\begin{abstract}
	\justifying
	Continual learning can incrementally absorb new concepts without interfering with previously learned knowledge. Motivated by the characteristics of neural networks, in which information is stored in weights on connections, we investigated how to design an Innately Forgetting-Free Network (IF2Net) for continual learning context. This study proposed a straightforward yet effective learning paradigm by ingeniously keeping the weights relative to each seen task untouched before and after learning a new task. We first presented the novel representation-level learning on task sequences with random weights. This technique refers to tweaking the drifted representations caused by randomization back to their separate task-optimal working states, but the involved weights are frozen and reused (opposite to well-known layer-wise updates of weights). Then, sequential decision-making without forgetting can be achieved by projecting the output weight updates into the parsimonious orthogonal space, making the adaptations not disturb old knowledge while maintaining model plasticity. IF2Net allows a single network to inherently learn unlimited mapping rules without telling task identities at test time by integrating the respective strengths of randomization and orthogonalization. We validated the effectiveness of our approach in the extensive theoretical analysis and empirical study.
\end{abstract}

}

\maketitle

\IEEEdisplaynontitleabstractindextext

%
\IEEEpeerreviewmaketitle

\IEEEraisesectionheading{\section{Introduction}\label{sec:introduction}}

%
%
%
%
\IEEEPARstart{C}{ontinual} learning (CL) is the ability to learn and remember a sequence of tasks, one at a time, with all data of the current task available but not previous or future tasks \cite{thrun1995lifelong}. However, even the current state-of-the-art deep neural networks (DNNs) would suffer from severe degeneration of recognition performance, known as the catastrophic forgetting phenomenon, when tasks are sequentially trained \cite{mccloskey1989catastrophic,goodfellow2013empirical}. The primary causes may be twofold. (i) DNNs are essentially connectionist models built by a target task, in which the acquired information is stored in connecting weights between adjacent layers. A previously learned mapping of an old task would be erased if one directly uses the standard gradient descent for learning new mappings. Consequently, the parameter-overwritten model would only remember the most recently learned task well. (ii) One of the most basic paradigms in machine learning is that test data are similarly distributed as the training data, called the independent and identically distributed (IID). However, Non-IID occurs when each task appears in sequences, e.g., input distributions are changing or/and unknown new categories emerge in future tasks. This leads to the crosstalk among labels at the output layer when it comes to decision-making.

Recently, CL has received increasing interest from the machine learning community \cite{parisi2019continual,delange2021clreview,wang2022dirichlet,yang2022dynamic,zhang2022epicker}. Prior work that is widely known on this topic falls into three families. First, replay-based approaches, e.g., exemplar replay, refer to straightforward maintaining a limited quantity of samples from each task seen and then blending them with that of a new task to retrieve the learned knowledge \cite{rebuffi2017icarl,NIPS2017GEM}. Though aided with exemplar buffers, a bounded memory budget typically carries inadequate task-specific knowledge, and the learning would still be biased toward the current mapping. Furthermore, these methods are infeasible when the data from previous tasks are not accessible due to memory constraints or privacy issues \cite{shokri2015privacy,ECCV2018MAS}. Their generative replay counterparts, alternatively, aim at yielding past observations in pseudo samples by using data generators \cite{van2018generative,ICCV2019LifelongGAN}. However, this transfers the stress from an assessed model to the generative network, making it demanding for the external module to precisely recover these old mappings \cite{wang2021triple}. Second, regularization-based approaches attempt to restrict weights without major changes by imposing penalties layer-by-layer \cite{PNAS2017EWC,ICML2017SI}. To this end, accumulative regularization terms are utilized to connection weights such that the updates only occur in constrained directions during sequential training \cite{ECCV2018MAS,WACV2020DMC}. Although these methods avoid storing data, the objectives of past tasks make a model rather inflexible to find the optimal parameters when the number of tasks is large \cite{chaudhry2021using}. Third, dynamic architecture-based approaches progressively expand a network where the newly assigned weights are exclusively in charge of new tasks \cite{CVPR2021EFT,xu2021adaptive,perkonigg2021dynamic}. This equates to allocating extra capacity for each task, and the memory complexity scales with the number of tasks \cite{AAAI2021PCL}. Interestingly, catastrophic forgetting cannot be intentionally attributed to insufficient network capacity because the same-sized networks can well accommodate multiple tasks when trained in an interleaved fashion or by using the data of all tasks seen so far \cite{mcclelland1995there,li2017learning}. 

Numerous follow-up works focused on the deficiencies observed above and further enhanced the CL performance in various ways, such as diversifying the limited exemplars \cite{tang2022learning,zhang2022memory}, constraining the magnitude of weight changes \cite{he2018overcoming,paik2020overcoming}, and compressing the extra expanded network size \cite{gao2022efficient,zhang2022self}, as summarized later in the related work section. However, we find that, in most CL approaches, the layer-wise updates of weights optimized by using stochastic gradient descent inevitably either interfere with old knowledge or obstruct learning for incoming tasks. This motivates our study on harnessing highly plastic weights and exploring the weight invariance during sequential training, and, most importantly, designing a simple network that is inherently robust against forgetting. Hence the name, innately forgetting-free network (IF2Net). The core idea is to retain the weights relative to each seen task untouched all the time, which is not well studied to the best of our knowledge. In each training session, IF2Net, with randomization and orthogonalization, provides a systematic solution for CL challenges, such that it potentially learns unlimited parallel input-to-output mappings under the task-agnostic setting.

The investigation presented can be regarded as an extended study of the existing CL families, providing an innovative idea and direction for future research on CL. The main contributions and strengths of this work include:

\begin{itemize}
\item A randomization-based representation-level learning technique is first developed to perfectly reproduce the representations from task sequences. We realized it via the forward passes unsupervised, relaxing the computational and memory requirements compared with the back-propagation algorithm. This allowed, for the first time, random weights (and biases) to be recycled to yield discriminative representations.

\item We introduced orthogonalization-based decision-making into the final output layer to break through the stability-plasticity dilemma. Consequently, the output weights are updated in the direction orthogonal to the subspace spanned by just obtained features from the penultimate layer, eliminating the inflexibility in the original layer-by-layer practice to the fullest. The parsimonious orthogonal space guarantees output weights unaffected relative to all tasks seen. In particular, we mathematically derived the orthogonal projection matrix for modulating gradients and provided proof to its learning without forgetting.

\item IF2Net is memory-efficient compared with replay and dynamic architecture-based methods because it requires no additional exemplars/parameters as new tasks arrive. Besides, the proposed method is compatible with most, if not all, CL methods due to its generality and, therefore, prompts those respective techniques in achieving new state-of-the-art results, e.g., together with the typical elastic weight consolidation (EWC) for a marginal boost. 

\item We demonstrated the effectiveness of our approach in the class-incremental learning scenario, evaluated by three metrics, and extensive experiments suggest that IF2Net consistently outperformed the prior state-of-the-art methods on MNIST, FashionMNIST, CIFAR-100, Stanford Dogs, CUB-200, and ImageNet-Subset benchmarks.
\end{itemize}

The remainder of this paper is organized as follows. Section 2 introduces related works. In Section 3, we present the CL setups and then detail the proposed IF2Net in Section 4. Section 5 conducts extensive experiments on six benchmark datasets to assess the superiority of our approach. Finally, we conclude our paper and suggest future research in Section 6.

\section{Related Work}
\textbf{Continual Learning.}
A learning paradigm, also known as lifelong/incremental learning, has been proposed to strike a balance between stability and plasticity \cite{NIPS2019RPS-Net,wen2021beneficial}. It places a single model in a dynamic environment where it must adapt a stream of tasks to make new predictions without forgetting \cite{jung2016less,li2017learning}. There are three main categories of recent CL algorithms \cite{parisi2019continual,NIPS2019RPS-Net,delange2021clreview,lesort2019continual}.

Replay-based approaches construct an exemplar buffer to save samples from previous tasks and train with current task data. As a representative, gradient episodic memory (GEM) obtains a positive transfer of knowledge to previous tasks by independently constraining the loss of each episodic memory non-increase \cite{NIPS2017GEM}. Incremental learning with dual memory (IL2M) utilizes both bounded exemplar images and past class statistics to rectify the network predictions \cite{ICCV2019IL2M}. Rainbow memory (RM) focuses on the classification uncertainty to select hard samples \cite{bang2021rainbow}. Furthermore, \cite{tang2022learning} exploits the abundant semantic-irrelevant information from unlabeled data to diversify the limited exemplars. However, replay-based approaches either generally deteriorate the performance with a smaller buffer size or are eventually not applicable to cases where data privacy should be taken into account \cite{shokri2015privacy,ECCV2018MAS}. Our approach does not rely on the exemplar buffer to tackle forgetting since it can inherently accommodate the previously learned mappings, thus satisfying the data privacy criterion.

Regularization-based approaches limit the model plasticity on learning important parameters of previous tasks by adding extra regularization terms. Elastic weight consolidation (EWC) is a trailblazer of this family, which employs a sequential Bayesian estimation and a Fisher information matrix to regularize weight updates \cite{PNAS2017EWC}. Representative approaches following it include synaptic intelligence (SI) \cite{ICML2017SI}, memory aware synapses (MAS) \cite{ECCV2018MAS}, and each network parameter is associated with the weight importance computed by different strategies. Under the Bayesian framework, another set of methods takes the posterior distribution of network parameters as the implicit penalty \cite{ICLR2018VCL,swaroop2019improving,loo2020generalized}. Although these methods address forgetting to some extent without storing past examples, they cannot perform satisfactorily in challenging settings or with complex datasets. By contrast, our approach could alleviate their demanding requirement of parameter optimization and take with them to set a new state of the art. 

Dynamic architecture-based approaches dynamically adapt network architectures by expansion or mask operation to accommodate knowledge needed for novel tasks. A progressive neural network (PNN) was proposed to gradually add new branches for all layers horizontally \cite{rusu2016progressive}. A random path selection network (RPS-Net) leverages parallel modules at each layer in which a possible searching space is formed to contain previous task-specific knowledge \cite{NIPS2019RPS-Net} for mitigating the uncontrollable enlargement of network size. The model parameters in additive parameter decomposition (APD) are decomposed into shared and task-specific ones using masks\cite{yoon2020scalable}. Efficient feature transformation (EFT) proposes a compact task-specific framework to overcome catastrophic forgetting \cite{CVPR2021EFT}. Per-class learning (PCL) employs a small-size multi-head setup where it branches an exclusive output layer for the classes learned so far via one-class learning \cite{hu2020hrn,AAAI2021PCL}. With prompt-based learning, \cite{wang2022learning} designs a query mechanism to dynamically look up a subset of task-relevant prompts by slightly introducing additional prompt parameters. On the contrary, our approach is characterized by a non-growing network, meaning that IF2Net could be constant in the number of parameters.

\textbf{Neural Networks with Random Weights.}
In conventional single-task learning, training a single-hidden layer feed-forward neural network with random weights (NN-RW) has demonstrated great potential in developing easy-implementation models, during which only the output weights need to be tuned \cite{pao1994learning,pao1992functional,suganthan2021origins}. Over decades of research, NN-RW is receiving more and more attention \cite{zhang2016visual,wang2017stochastic,TNNLS2017BLS,wang2021compact,huang2022graph}. Readers may refer to \cite{zhang2016comprehensive,scardapane2017randomness,cao2018review,gong2021research} for a comprehensive review of randomized learning techniques. Here, a crucial technical issue lies in randomly assigning the input weights (and biases) such that the randomized learner model shares a universal approximation capability in either a deterministic or nondeterministic manner. Randomly assigning input weights (and biases) from a fixed probability distribution sounds simple and fast and has been commonly adopted by random vector functional-link network (RVFLN) \cite{pao1994learning}, extreme learning machine (ELM) \cite{huang2006extreme}, stochastic configuration network (SCN) \cite{wang2017stochastic}, and broad learning system (BLS) \cite{TNNLS2017BLS}. The resulting neural networks with fixed random hidden parameters may have a universal approximation capability in a probability sense, provided that an appropriate distribution (e.g., a uniform range $[r,-r], r>0$) is properly set in advance \cite{li2017insights,dai2019BSCN,dai2021hybrid}. This means a certain “support range” leading to a randomized universal approximator exists and is data-dependent for a given task. Inspired by this technical point, we aimed to extend the fixed random hidden parameters to overcome catastrophic forgetting. However, no work has been presented so far to CL with the help of a randomized learning technique due to its random nature (performance fluctuation), shallow architecture (inferior representation), and IID condition (same as DNNs). This motivated our investigation on building stable representation-level learning with randomized multiple-hidden layers on task sequences to fill the gap. To this end, we would significantly improve NN-RW, as explained in Sections \ref{Sec_RRL} and \ref{Discussion_1}.

\textbf{Orthogonal Gradient Descent.}
Lately, some literature has started to consider the orthogonal gradient descent strategy, a back-propagation variant \cite{singhal1989training}, by keeping the input-to-output mappings fixed \cite{he2018overcoming,NMI2019OWM,li2021gopgan}. To this end, network weights are sequentially optimized in the orthogonal space of each linear layer's previously learned inputs/features. The pioneering works include conceptor-aided backpropagation (CAB) \cite{he2018overcoming} and orthogonal weights modification (OWM) \cite{NMI2019OWM}, which resorts to different mathematical theories to compute the orthogonal subspace approximately. Concretely, an orthogonal projection matrix in each layer of networks is constructed to preserve the learned knowledge. Furthermore, the gradients in neural networks are projected to the orthogonal direction of all previously learned features by the orthogonal projection matrix during learning a new task. Unlike CAB and OWM, we only applied it to the final output layer for decision-making while simplifying the form of the orthogonal projector. Detailed differences and strengths will be discussed in Sections \ref{Sec_ODM}, \ref{Discussion_2}, and \ref{Discussion_3}.

\section{Continual Learning Setup}
\textbf{Definition.}
We focused the scope of our work on the most prevalent yet challenging class-incremental learning scenario, namely CIL \cite{hsu2018re,masana2022class}. A model in this scenario solves each task seen so far and infers which task it would be encountered with, which includes the common real-world problem of incrementally learning new classes of objects. Mathematically, let us start by defining the CIL. Given the supervised learning datasets $\bm{D}_t=\{(\bm{X}_t,\bm{Y}_t)|\bm{X}_t\in \mathbb{R}^{N_t\times M_t}, \bm{Y}_t\in \mathbb{R}^{N_t\times C_t}\}$ of task $t$ ($t=1,2,\dots$), where $\bm{X}_t$ is the input, $\bm{Y}_t$ is the label, $N_t$ is the number of samples, $M_t$ and $C_t$ are the dimensions of input data and output classes, respectively. Assumed a model $f(\bm{X}_{t-1};\bm{\theta}_{t-1})$~$(t\geq2)$ trained on previous task(s), parameterized by its connection weights $\bm{\theta}_{t-1}$, the objective is to train an updated model $f(\bm{X}_t;\bm{\theta}_t)$ that can accommodate the newly emerging $C_t$ classes \cite{van2019three}, that is, $f(\bm{X}_T;\bm{\theta}_T)$ needs to remember how to perform well on the cumulative $C=\sum_{t=1}^{T}{C_t}$ classes. Meanwhile, the data of previous tasks $\bm{D}_{t}$ ($t=1,2,\dots, T-1$) for $f(\bm{X}_T;\bm{\theta}_T)$ is unavailable except for the replay-based approaches/baselines, e.g., a handful of samples in exemplar buffers can be revisited. At test time, the resulting model would be fed samples from any of tasks 1 to $T$ without telling the task identifier/oracle and evaluated by the standard CL metrics formulated as follows.

\textbf{Metrics.}
We monitored three statistics to reflect the quality of CL methods (all higher is better). First, the \textit{final average test accuracy} (ACC) reported in the majority of literature is defined as:

\begin{equation}\label{ACC}
\text{ACC}=\frac{1}{T}\sum_{t=1}^T R_{T,t}
\end{equation}

\noindent where $R_{T,t}$ means the test classification accuracy of a model on task $t$ after training on task $T$. ACC measures how the average performance of the model degrades as it learns new tasks and directly shows the test performance of a model on all tasks seen so far. Second, \textit{backward transfer} (BWT) \cite{NIPS2017GEM} is defined as the difference in accuracies between when a task is first trained and after training on the final task, averaged over all tasks:

\begin{equation}\label{BWT}
\text{BWT}=\frac{1}{T-1}\sum_{t=1}^{T-1} R_{T,t}-R_{t,t}
\end{equation}

\noindent It indicates a model's ability in knowledge retention. As a result, the lower BWT value corresponds to catastrophic forgetting, while BWT = zero is considered forgetting-free. Third, \textit{forward transfer} (FWT) \cite{NIPS2020FROMP} defines the average improvement in accuracy on a new task over an independently trained model on that task, which is expressed as:

\begin{equation}\label{FWT}
\text{FWT}=\frac{1}{T-1}\sum_{t=2}^{T} R_{t,t}-R_t^{ind}
\end{equation}

\noindent where $R_t^{ind}$ is the test classification accuracy of an independent model trained only on task $t$. FWT estimates how well a model uses previously learned knowledge to improve classification accuracy on newly seen tasks. Indeed, BWT and FWT complement ACC, e.g., if two models have similar ACC, the preferable one shares the larger BWT and FWT.

%
%
%
%

\section{Innately Forgetting-Free Network} \label{Sec_IF2Net}
\subsection{Overview}
This section explains how we defied the interference among sequential tasks by designing an innately forgetting-free network termed IF2Net. To this end, we incorporated dual projections (randomization and orthogonalization) into IF2Net such that it could keep the weights relative to each seen task untouched before and after learning a new task. In particular, the CL process is primarily mediated by the randomization-based representation-level learning in several hidden layers and the orthogonalization-based decision-making in the final output layer. An overview is shown in Fig. \ref{Fig_IF2Net}. To obtain a good grasp of the learner model produced by IF2Net, theoretical verification of two key properties, such as convergence analysis of representation approximation and proof of learning-without-forgetting decision, are provided with rigorous mathematical deductions, showing the feasibility of the proposed method under certain conditions. After this, we will highlight the prime and original contributions with much in-depth discussion.

\begin{figure}[htbp]
\centering
\includegraphics[width=1.0\columnwidth]{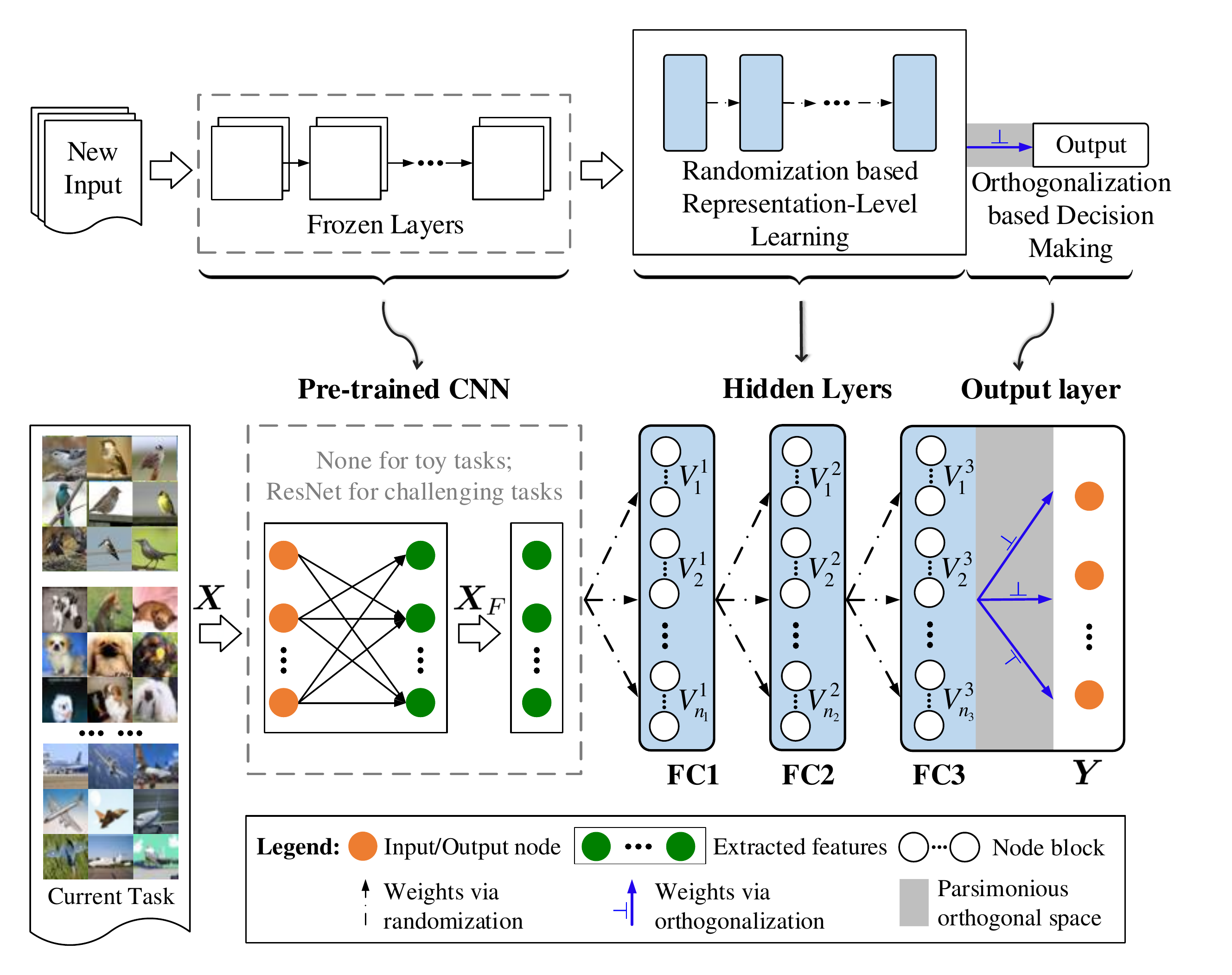}
\caption{The proposed pipeline (top panel) and an IF2Net architecture used in this study (bottom panel). IF2Net takes a new input image and passes it through frozen layers of a fixed pre-trained CNN to obtain rich representations. It then proceeds to randomization-based representation-level learning with reused random weights (and biases) connecting FC1-FC3, termed hidden layers. The output weight in the output layer is updated in the parsimonious orthogonal space spanned by discriminative features obtained from FC3 to make learning-without-forgetting decisions. Dashed boxes indicate components that are not required (tasks can be sequentially fed with the original inputs $\bm{X}$ or extracted features $\bm{X}_F$ when it is necessary).}
\label{Fig_IF2Net}
\end{figure}

\subsection{Randomization based Representation-Level Learning} \label{Sec_RRL}
The motivation for introducing randomization-based representation-level learning is threefold. First, directly retraining a network modifies its connecting weights away from optimal solutions to old tasks with new learning overwriting existing representations. Second, even if additional “memories” (such as exemplar buffers, model weights, and regularization matrices) of past tasks are used to retrieve previously learned knowledge, the layer-wise updates of weights optimized by stochastic gradient descent either interfere with old knowledge or obstruct learning for incoming tasks. Third, NN-RW with fixed random hidden parameters can share a universal approximation property given a task \cite{needell2020random}, that is, one can deemphasize the importance of weights by randomly assigning the input weights (and biases). Therefore, instead of optimizing hidden-layer weights of a fixed network, we instead emphasize the required activations (representations) that perform well over a sequence of tasks. This technique challenges common practice to push the drifted representations caused by randomization back to their separate task-optimal working states, during which the involved random weights (and biases) are reused. The details are as follows.

\subsubsection{Representation-level learning process}
Before the representation-level learning starts, IF2Net leverages the representative features from a pre-trained CNN (although it is not required); however, instead of tuning the parameters during the CL process, IF2Net keeps the pre-trained model fixed since the early layers of CNN can be highly transferable \cite{yosinski2014transferable,hayes2020remind}, and incrementally learns a set of rich representations from incoming tasks when used, resembling what \cite{NMI2019OWM,hayes2020remind,AAAI2021PCL,wang2022learning,wu2022class} has been performed. Hence, the input of IF2Net can be the original samples $\bm{X}$ from toy tasks (MNIST and FashionMNIST) or the extracted features $\bm{X}_{F}$ from high-dimensional tasks (Stanford Dogs, CUB-200, CIFAR-100, and ImageNet-Subset). We used $\bm{X}$ to denote both variables to simplify the notation.

The topology used in representation-level learning comprises several commonly used fully connected (FC) layers, but the training process is entirely different from that in previous work. With the help of randomly initialized weights (and biases) in the hidden layers, the inputs/features are transformed into a random representation space, which can help exploit hidden information among the training samples. In particular, each layer $l$ $(l=0,1,\dots,L-1)$ in this subpart randomly allocates $S_l$ nodes that are composed of $n_l$ groups of node blocks, with each group $s_l$ nodes to diversify their counterparts of output representations. In the following, we used a simple three-layer fully connected implementation (see the Hidden Layers located in the bottom panel of Fig. \ref{Fig_IF2Net}) to describe the representation-level learning process. First, we defined the forward function of $j$th ($j=1,2,\dots,n_{l+1}$) node block at layer $l+1$ as follows:

\begin{equation}\label{Vl}
\bm{V}^{l+1}_j=\sigma(\tilde{\bm{V}}^l\bm{W}^l_j+1_{N_t}\bm{b}^l_j)
\end{equation}

\noindent where $\tilde{\bm{V}}^l\in\mathbb{R}^{N_t \times S_l}$ is the tweaked activations (representations) from preceding layer $l$, matrix $\bm{W}^l_j \in\mathbb{R}^{S_l \times s_{l+1}}$ and row vector $\bm{b}^l_j\in\mathbb{R}^{1\times s_{l+1}}$ are the randomly assigned weights within a proper scope setting (e.g., $[-1, 1]$) \cite{dai2019BSCN}, $1_{N_t} \in \mathbb{R}^{N_t}$ is the column vector with every entry being one, and $\sigma(\cdot)$ is the activation function at each layer, respectively. Note that $\tilde{\bm{V}}^0=\bm{X}_t$ and $\bm{\theta}=(\bm{W}^l_j,\bm{b}^l_j)$. However, the above single forward passes potentially lead to drifted representations $\bm{V}^{l+1}_j$ ($j=1,2,\dots,n_{l+1}$) caused by randomization \cite{wang2017stochastic,dai2019BSCN}, as depicted in Fig. \ref{Fig_RRL}.

\begin{figure}[tbp]
\centering
\includegraphics[width=1.0\columnwidth]{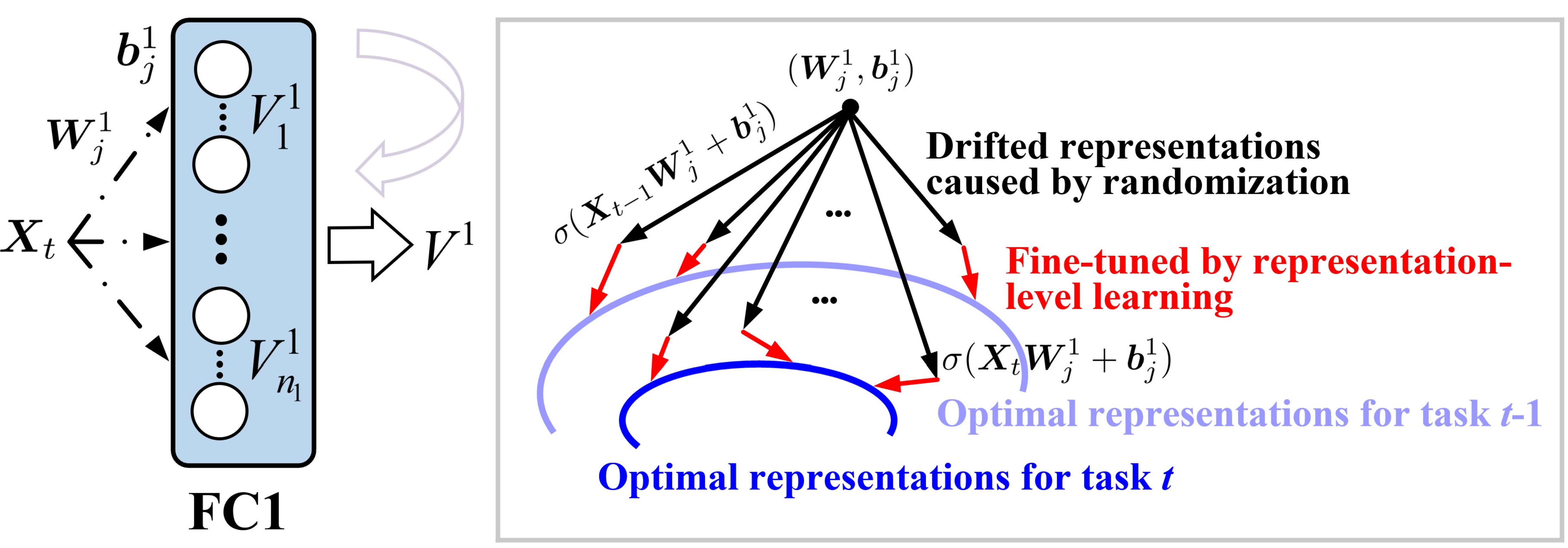}
\caption{ Randomization-based representation-level learning process. Take FC1 as an example. When $\bm{X}_t$ ($t=1,2,\dots$) is present, we first used the randomly assigned weights (dotted arrow in black, left) to potentially yield inferior responses called drifted representations (solid arrow in black, right). Then, they were tweaked by unsupervised fine-tuning back to their separate task-optimal working states (solid arrow in red, right). The forward representation-level learning can perform well on task sequences without weight modification. Best viewed in color.}
\label{Fig_RRL}
\end{figure}

To address this issue, we formulated it as an optimization problem, and the objective function is defined as:

\begin{equation}\label{Theta}
\tilde{\bm{\theta}} \in \argmin_{(\tilde{\bm{W}}^l_j,\tilde{\bm{b}}^l_j)} \Vert\bm{V}^{l+1}_j\tilde{\bm{W}}^l_j+1_{N_t}\tilde{\bm{b}}^l_j-\tilde{\bm{V}}^l\Vert^2_2+\Vert\tilde{\bm{W}}^l_j\Vert_{1} + \Vert\tilde{\bm{b}}^l_j\Vert_{1}
\end{equation}


\noindent where $\tilde{\bm{\theta}}=(\tilde{\bm{W}}^l_j,\tilde{\bm{b}}^l_j)$ is the intermediate solution, $\Vert\cdot\Vert_1$ and $\Vert\cdot\Vert_2$ are the $L_1$ and $L_2$ norm, respectively. Conceptually, (\ref{Theta}) can serve as a feature processor in hierarchical neural networks, where the randomized response $\bm{V}^{l+1}_j$ is used to approximate the input $\tilde{\bm{V}}^l$ by minimizing the reconstruction errors \cite{vincent2008extracting,TNNLS2017BLS}. Then, the drifted representation $\bm{V}^{l+1}_j$ can be tweaked by replacing raw $\bm{\theta}$ with fine-tuned $\tilde{\bm{\theta}}$, that is, $\tilde{\bm{V}}^{l+1}_j=\sigma(\tilde{\bm{V}}^l\tilde{\bm{W}}^l_j+1_{N_t}\tilde{\bm{b}}^l_j)$. Hence, this forward pass can manage to slightly fine-tune each bunch of randomized responses $\bm{V}^{l+1}_j$ ($j=1,2,\dots,n_{l+1}$) in an unsupervised manner but freeze the involved weight $\bm{\theta}=(\bm{W}^l_j,\bm{b}^l_j)$, such that the drifted representations of each task could always be dragged back to their separate/initial task-optimal working states $\tilde{\bm{V}}^{l+1}_j$ (see Fig. \ref{Fig_RRL}) and thus flawlessly reproduce representations of each task seen so far, as explained later in Section \ref{sec:4.2.2}. Meanwhile, we can recursively obtain the tweaked activations $\tilde{\bm{V}}^{l+1}=[\tilde{\bm{V}}^{l+1}_1,\tilde{\bm{V}}^{l+1}_2,\dots,\tilde{\bm{V}}^{l+1}_{n_{l+1}}]$ in layer $l+1$ by concatenating all $n_{l+1}$ groups of node blocks. We will validate the efficiency of multiple node blocks in Section \ref{Node_blocks}. The above randomization-based representation-level learning demonstrates that reused random weights (and biases) are feasible for obtaining optimal representations from task sequences.

The intuition behind representation-level learning is that most CL methods are driven by weight-level optimization, which is not necessarily the underlying way that connectionist models distinguish and remember knowledge. The involved weights in hidden layers can be transitional or disposable. In this sense, rather than judging CL performance with globally optimized weights, we can measure it with the current representations. Hence, the latter is comparable to some approaches for single-task learning. Consider the FashionMNIST as an example. Fig. \ref{Visualization} depicts the t-SNE visualization \cite{van2008visualizing} of the raw sample space and the final projected representations that correspond to the outputs of FC3, which indicates the good potential of randomization-based representation-level learning. The distinguishable representations obtained from a sequence of tasks are then used for decision-making in the final output layer.

\begin{figure}[tbp]
\centering
\subfloat[]{\includegraphics[height=0.85in]{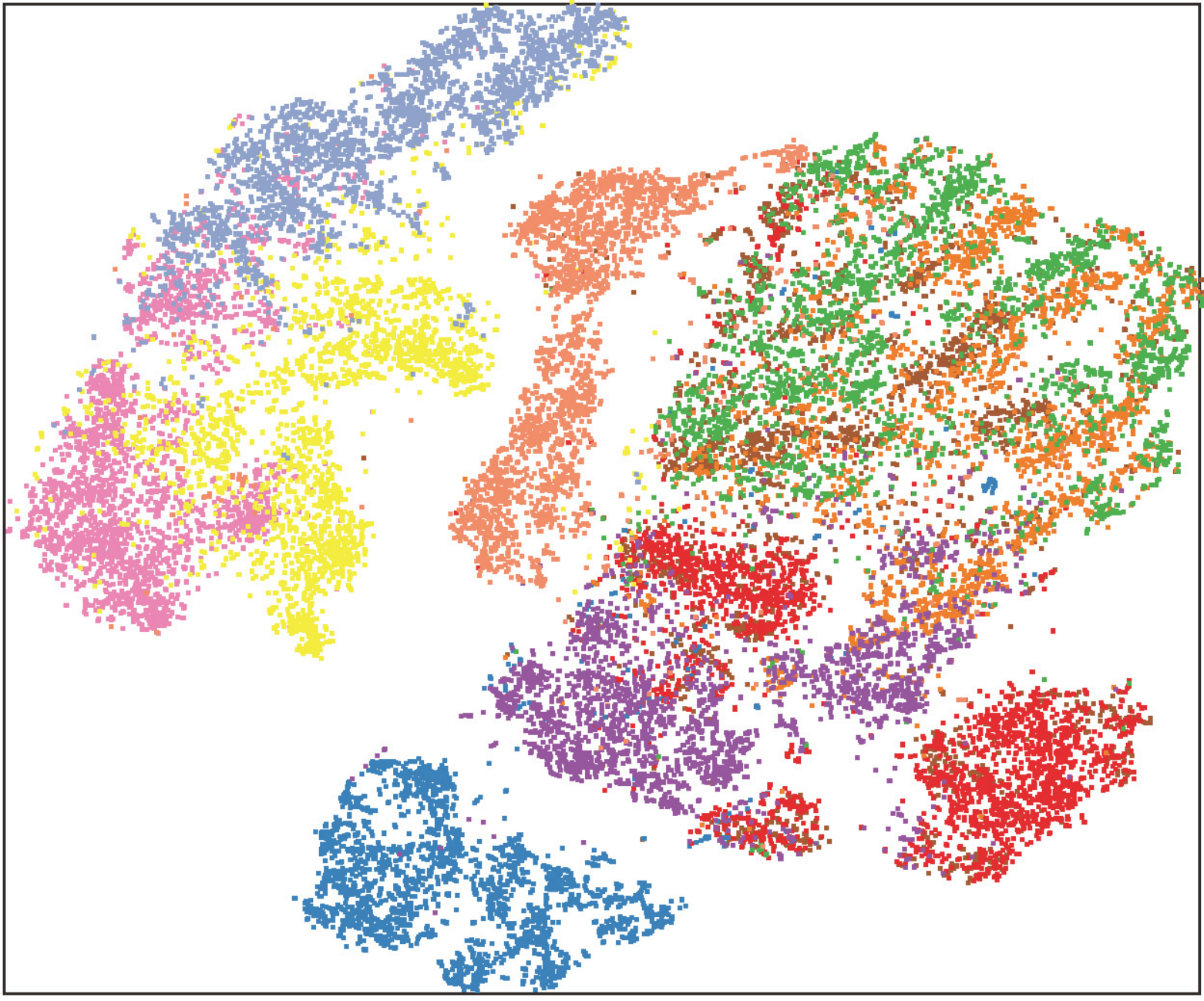}
	\label{Raw_split5}}
\hfil 
\subfloat[]{\includegraphics[height=0.85in]{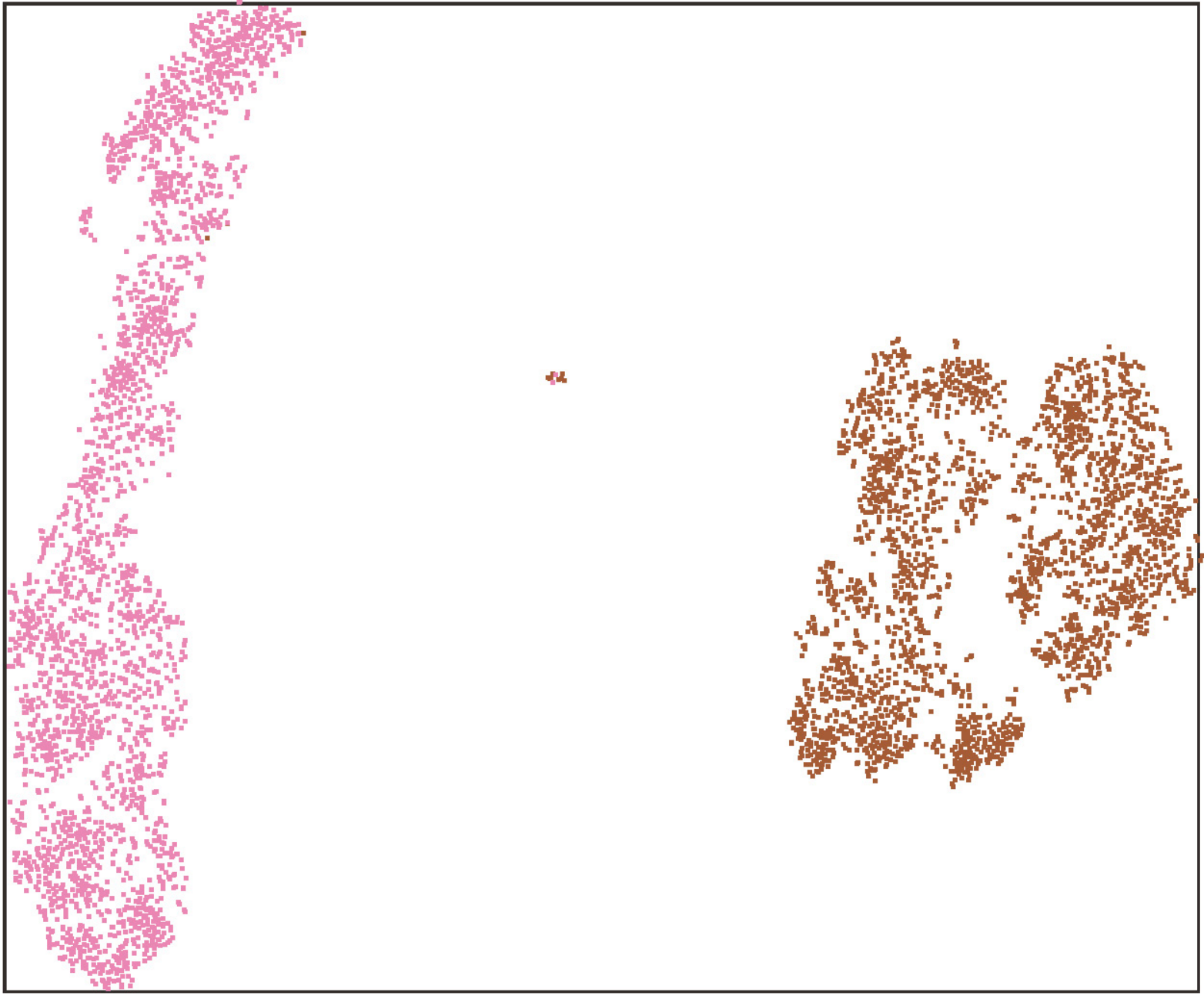}
	\label{Task1}}
\hfil
\subfloat[]{\includegraphics[height=0.85in]{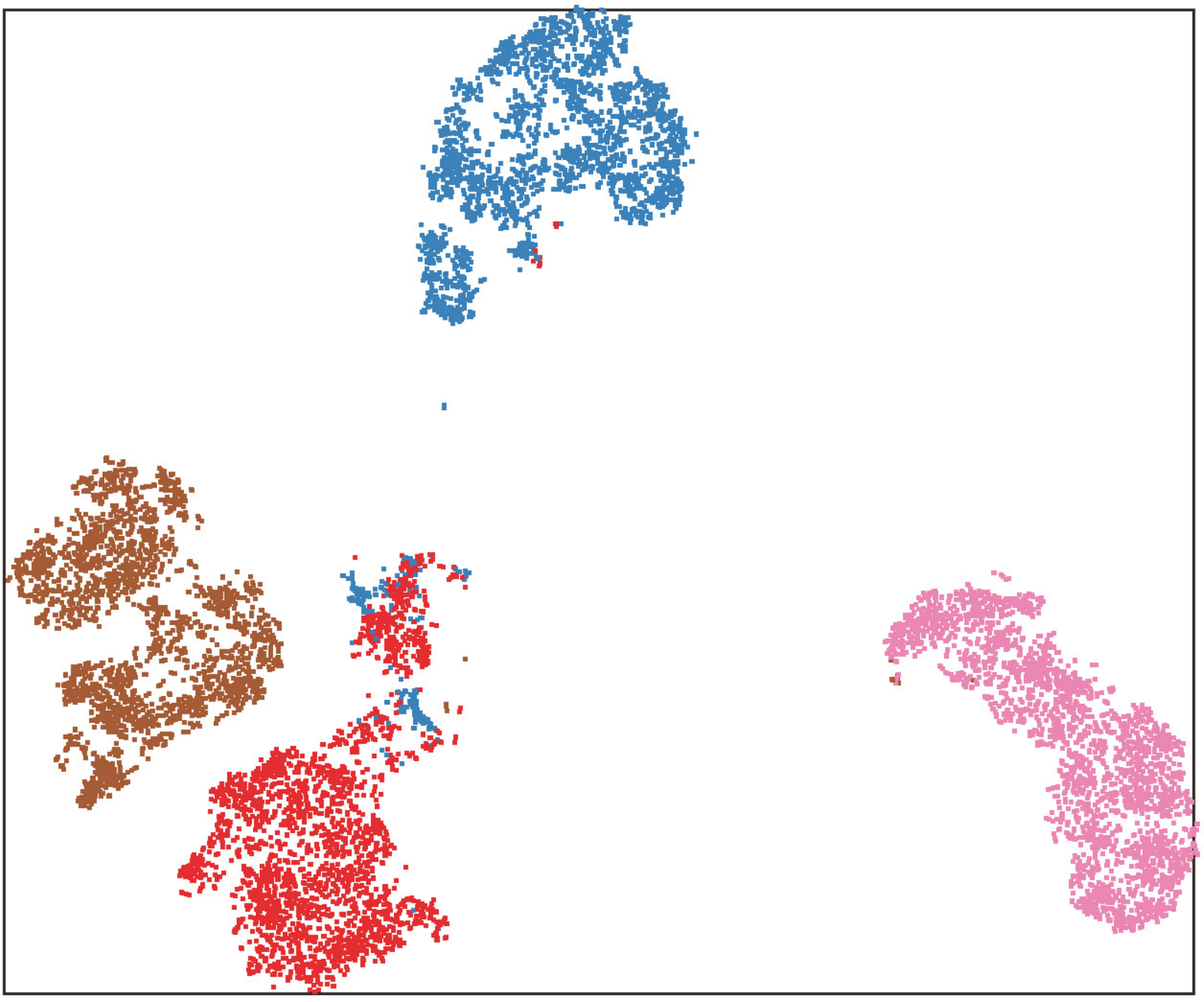}
	\label{Task12}}
\hfil
\subfloat[]{\includegraphics[height=0.85in]{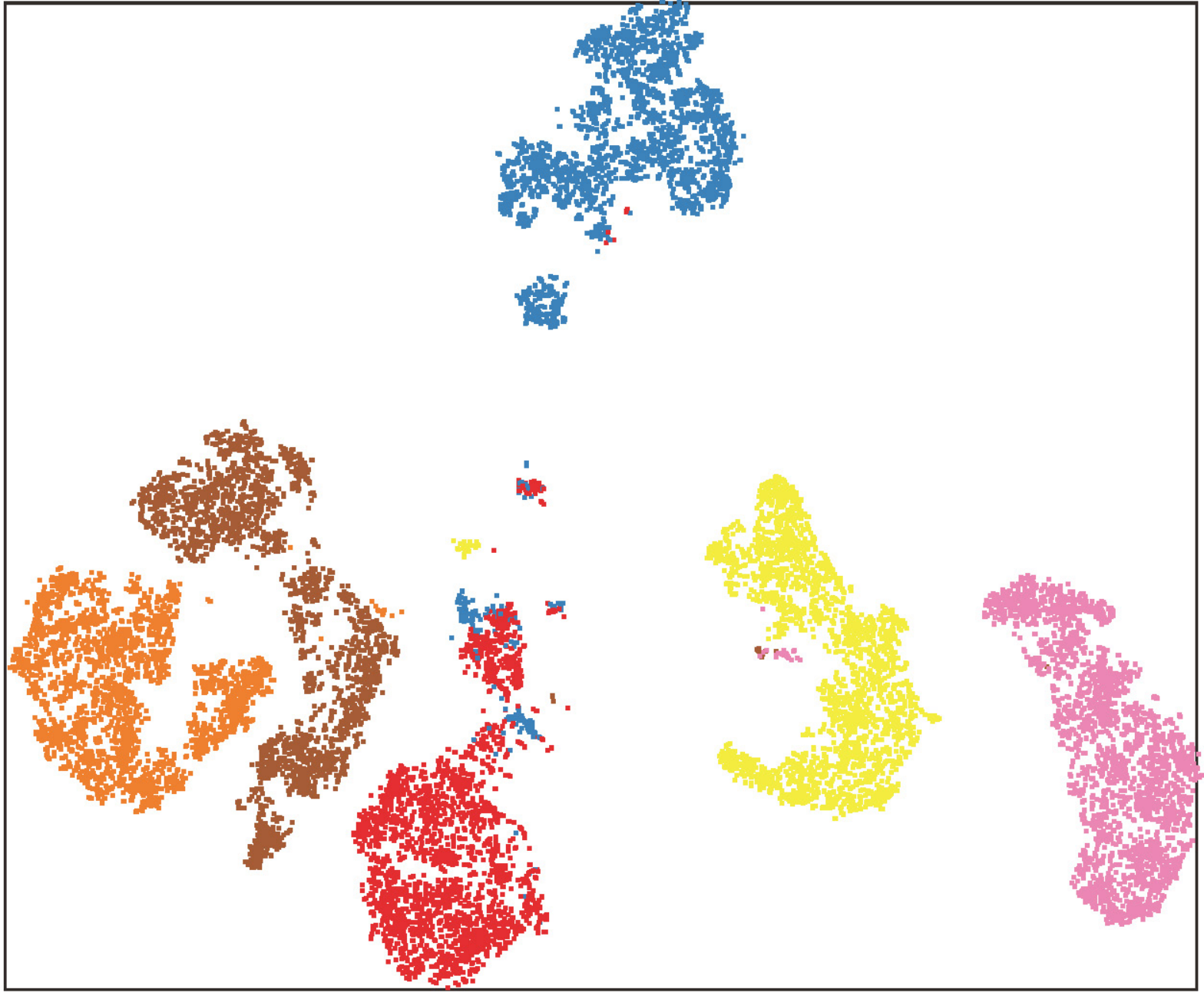}
	\label{Task123}}
\hfil
\subfloat[]{\includegraphics[height=0.85in]{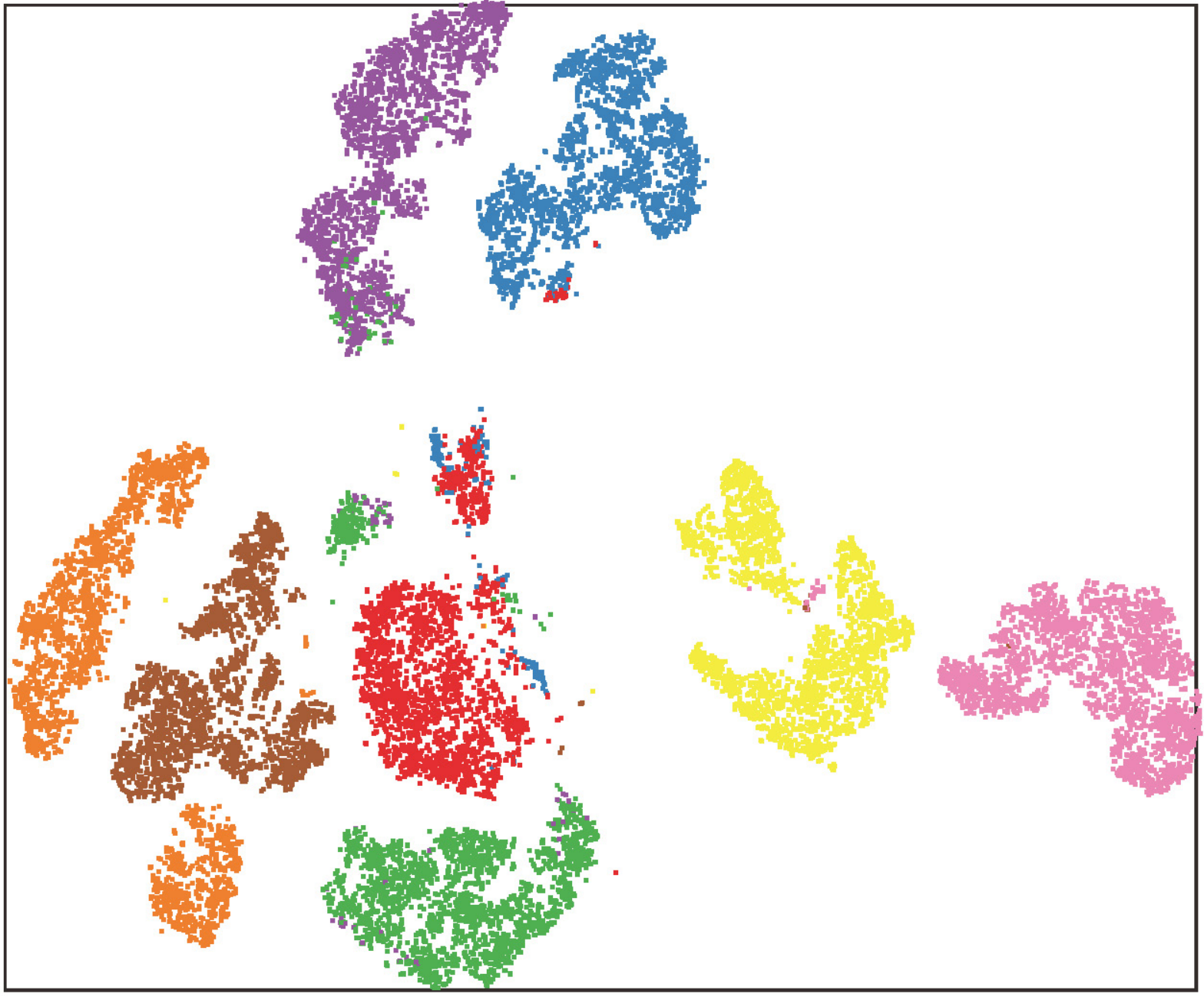}
	\label{Task1234}}
\hfil
\subfloat[]{\includegraphics[height=0.85in]{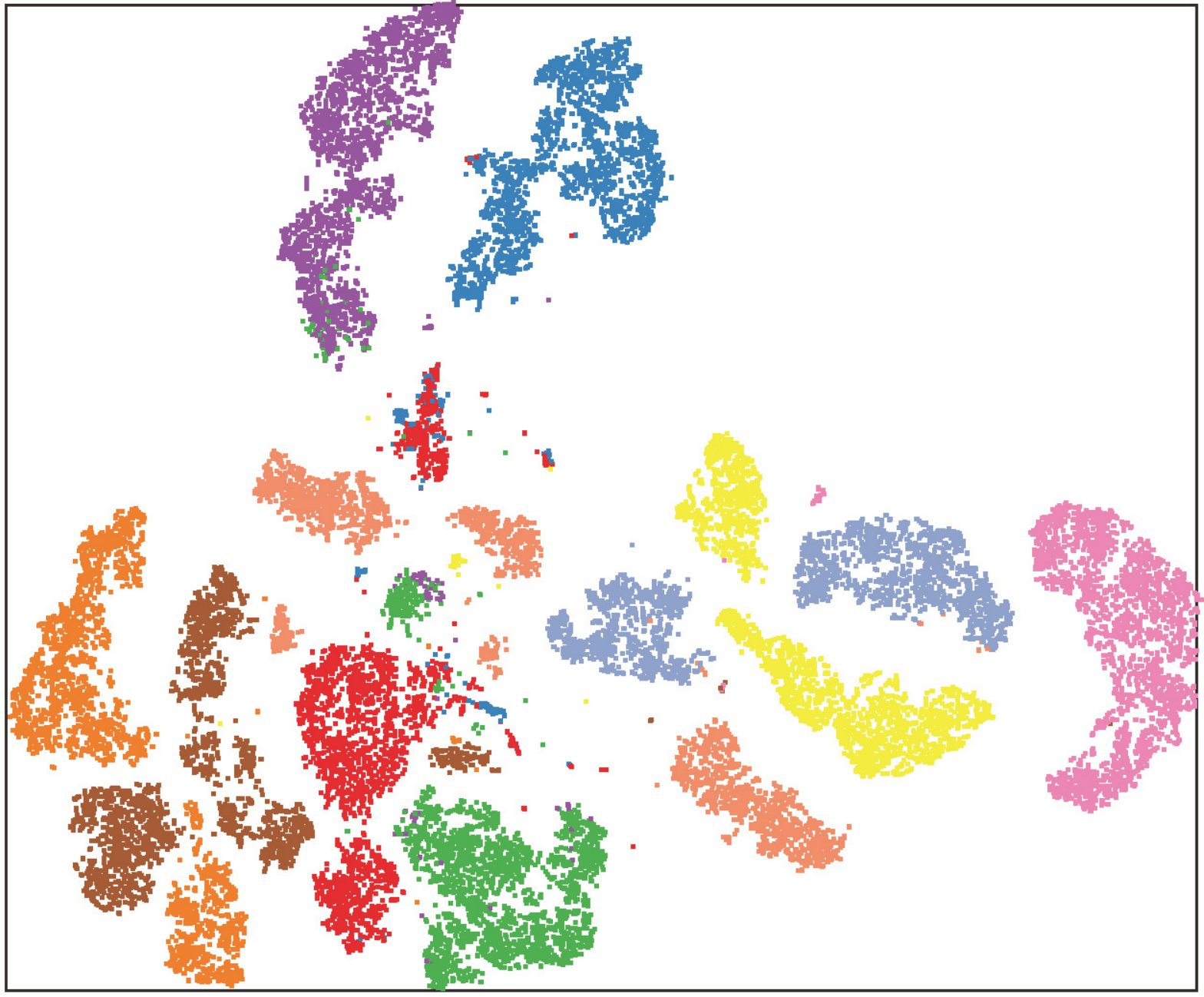}
	\label{Task12345}}
\caption{t-SNE visualization based on Split FashionMNIST benchmark. Each color represents a class, and two classes are incrementally fed into the hidden layers as an independent task. In this task sequence, training data of all the classes together are first visualized as a reference, followed by the tweaked representations after incrementally learning two classes per session. (a) Mixed raw sample space. (b)-(f) Well-clustered representation space that corresponds to the projected outputs of FC3. IF2Net can yield distinguishable representations from a sequence of tasks for the final decision-making. Best viewed in color.}
\label{Visualization}
\end{figure}

\subsubsection{Convergence analysis} \label{sec:4.2.2}
Here, we provided an iterative algorithm based on the discrete-time projection neural network \cite{xu2018discrete} to solve the optimization problem in (\ref{Theta}) and then provided a theoretical analysis to explain whether the proposed representation-level learning converges under certain conditions. 

Let $g(\bm{u})$: $\mathbb{R}^n \rightarrow \mathbb{R}^n$ be a projection operator defined on a closed convex set $\bm{\Omega}=\{\bm{u}\subseteq\mathbb{R}^n: l_i\leq u_i\leq h_i, i=1,2,\dots,n\}$ then,

\begin{equation}\label{g_ui}
g(u_i)=\left\{
\begin{aligned}
	h_i \quad &u_i>h_i\\
	u_i \quad &l_i\leq u_i\leq h_i\\
	l_i \quad &u_i<l_i\\
\end{aligned}
\right.
\end{equation}

\noindent which is a piecewise linear function. The property for  problem-solving is presented in the following lemma:

\begin{lemma}\cite{liu2015l_}\label{Lemma_1}
Assume there is a full-rank matrix $\bm{A}$. For any $\bm{x}\in\mathbb{R}^n$, $\bm{A}\bm{x}=\bm{b}$ if and only if $\bm{Z}\bm{x}=\bm{q}$, where $\bm{Z}=\bm{A}^\mathrm{T}(\bm{A}\bm{A}^\mathrm{T})^{-1}\bm{A}$ and $\bm{q}=\bm{A}^\mathrm{T}(\bm{A}\bm{A}^\mathrm{T})^{-1}\bm{b}$.
\end{lemma}

With these preliminaries, we can generally rewrite the objective function in (\ref{Theta}) as:

\begin{equation}\label{Theta_solution}
\arg\min_{\bm{x}^*}: \frac{1}{2}\Vert\bm{Z}\bm{x}^*-\bm{q}_t\Vert^2_2+\lambda\Vert\bm{\bm{x}^*}\Vert_1
\end{equation}

\noindent where $t \in\mathbb{N}^+$ is the task index, $\bm{q}_t$ is the desired input vector, $\bm{Z}$ is the drifted representations, $\lambda$ is a positive scalar weight, and $\bm{x}^*$ is the surrogate associated with $\tilde{\bm{\theta}}$. The convergence property of the proposed representation-level learning can be stated in the following theorem:

\begin{theorem}\label{Theorem_1}
For $t \in\mathbb{N}^+$, $\bm{x}^*\in\mathbb{R}^n$ is an optimal solution to problem (\ref{Theta_solution}) if and only if there exits $\bm{y}^*\in\mathbb{R}^n$ such that $\bm{x}^*$ and $\bm{y}^*$ satisfy:

\begin{equation}\label{Theorem_condition}
	\left\{
	\begin{aligned}
		&\bm{Z}\bm{x}^*-\bm{q}_t+\lambda\bm{y}^*=0\\
		&\bm{y}^*=g(\bm{y}^*+\bm{x}^*).\\
	\end{aligned}
	\right.
\end{equation}
\end{theorem}

\begin{proof}\label{pfoof}
$\bm{x}^* \in\mathbb{R}^n$ is an optimal solution to problem (\ref{Theta_solution}) if and only if there exits $\bm{y}^*\in\mathbb{R}^n$ such that $\bm{x}^*$ and $\bm{y}^*$ satisfy the following condition:

\begin{equation}
	\bm{Z}\bm{x}^*-\bm{q}_t+\lambda\bm{y}^*=0,~t=1,2,\dots
\end{equation}

where $\bm{y}^*=[y_1^*,y_2^*,\dots,y_n^*]^\mathrm{T}$ is a subgradient of $\Vert\bm{\bm{x}^*}\Vert_1$ with each component $y_j^*$ $(j=1,2,\dots,n)$ defined as: 

\begin{equation}\label{y_star}
	y_j^*\left\{
	\begin{aligned}
		&=1 \quad &x_j^*>0\\
		&\in[-1,1] \quad &x_j^*=0\\
		&=-1 \quad &x_j^*<0\\
	\end{aligned}
	\right.
\end{equation}

Then, $\bm{y}^*$ satisfies: 

\begin{equation}
	\bm{y}^*=g(\bm{y}^*+\bm{x}^*)
\end{equation}

where $g(\cdot)$ is a projection operator from $\mathbb{R}^n \rightarrow [-1,1]^n$ as defined in (\ref{g_ui}) with $l_i=-1$ and $h_i=1$. $\hfill\square$
\end{proof}

Without loss of generality, the solution to (\ref{Theta_solution}) can be iteratively described as follows \cite{xu2018discrete}:

\begin{equation}\label{Theta_iteratrion}
\left\{
\begin{aligned}
	&\bm{x}(k+1)=\bm{x}(k)-\gamma(\bm{Z}\bm{x}(k))-\bm{q}_t+\lambda g(\bm{y}(k)+\bm{x}(k))\\
	&\bm{y}(k+1)=g(\bm{y}(k)+\bm{x}(k+1))\\
\end{aligned}
\right.
\end{equation}

\noindent where $\gamma > 0$ is the gain of the network, $\bm{x}(k)$ is the state vector, and $\bm{y}(k)$ is the output vector, respectively. 

The aforementioned equation theoretically analyzed how the proposed representation-level learning technique can converge the randomized responses of different tasks to their optimum values. We presented pseudo-codes of the representation-level learning process in \textbf{Algorithm \ref{alg_1}} to clarify the algorithm implementation. In addition, we noted that it achieves the same performance if (\ref{Theta_solution}) is solved in an alternative manner, such as by extending the fast iterative shrinkage-thresholding algorithm (FISTA) \cite{beck2009fast} to representation-level learning.

\begin{algorithm}[htb]
\caption{Randomization based Representation-Level Learning}
\label{alg_1}
\begin{algorithmic}[1]
	\Require 
	Datasets $\bm{D}_t=\{(\bm{X}_t,\bm{Y}_t)\}$ of task $t$; 
	A base network with $L$ hidden layers;
	Number of node blocks $n_l$ and nodes in each block $s_l$;
	The positive scalars $\lambda$ and $\gamma$. 
	
	\Ensure 
	Resulting representations of $L$th layer $\bm{V}^L$.
	
	\State \textit{\# initialize parameters from a uniform range}
	\For{$l=0, 1,\dots,L-1$}
	\For{$j=1, 2,\dots,n_{l+1}$}
	\State Randomly assign weights $\bm{W}^l_j$ and biases $\bm{b}^l_j$;  
	\EndFor
	\EndFor
	
	\State \textit{\# representation-level learning on sequential tasks}
	\For{$t=1, 2,\dots,T$}
	\For{$l=0, 1,\dots,L-1$}
	\State Obtain drifted representations $\bm{V}^{l+1}$ based on (\ref{Vl});
	\State Tweak drifted representations based on (\ref{Theta});		
	\EndFor
	\EndFor
\end{algorithmic}
\end{algorithm}

\subsection{Orthogonalization based Decision-Making} \label{Sec_ODM}
In the CL context, the final output layer, dubbed a classifier, must accommodate the changing input distributions or/and new unknown categories emerging in future tasks, which is vital to both previously learned and incoming tasks. We implemented an orthogonalization algorithm that theoretically achieved zero-forgetting decisions by making updated weights that do not disturb old knowledge while maintaining model plasticity. In the following, we explained why we expect to form a parsimonious orthogonal space for decision-making, how we leveraged orthogonalization and implemented output weight updates efficiently.

\subsubsection{Parsimonious orthogonal space}
Using the orthogonal gradient descent method, network weights can be sequentially optimized in the orthogonal space of all previously learned inputs/features for each FC layer \cite{he2018overcoming,NMI2019OWM,li2021gopgan}. However, because the trained weights are bounded to the counterpart space formed by the previous tasks, minimizing the mean squared error or cross-entropy loss would lead to two major problems:

First, layer-wise orthogonalization entails massive computational and memory requirements that are proportional to the number of network layers. It needs to store and update the projection matrix for each layer, accounting for why the orthogonal gradient descent algorithm sometimes requires sufficient memory to maintain space, especially in models with many network layers and nodes. More importantly, it is more vulnerable to increasing the generalization bound error during sequential training, which adversely affects the knowledge retention of previous tasks.

Second, layer-wise orthogonalization makes it difficult to find the optimal parameters when CL tasks are challenging. A standard gradient descent usually suffers from slow convergence and traps at a local minimum, let alone its orthogonal implementation, by updating the network weights along the constrained direction. Particularly, with the backpropagation of the entire network layer-by-layer, the constraint is likely to be greatly amplified, degenerating the new model’s ability to retain knowledge. 

Unlike these methods, we only applied it to the output weights for decision-making and refined the form of the orthogonal projection matrix, as presented in the subsequent section. Hence, the nature of parsimony mainly referred to two aspects. (i) For an individual task, it only works on the final output layer, as opposed to the commonly used layer-wise operations. This simplification is performed well by representation-level learning. In this way, our approach relaxed the computational and memory requirements, as we do not orthogonalize the weights of reused neurons in hidden layers and eliminated the inflexibility to the fullest. (ii) For the sequential tasks, it only cares for one orthogonal projector based on the representations of the penultimate layer. By contrast, layer-wise orthogonalization requires the projector at each layer to be delivered and updated during learning a new task. As a result, the parsimonious orthogonal space would be better for guaranteeing that the output weights are impervious relative to all the tasks seen.

\subsubsection{Orthogonal projection matrix}
We first presented the decision-making process with proof of its learning property without forgetting, followed by mathematically deriving the orthogonal projection matrix for modulating gradients. The details are as follows.

When finishing representation-level learning on task $t$ $(t=1,2,\dots)$, IF2Net can sequentially obtain distinguishable representations at FC3 termed $\tilde{\bm{V}}^3$. Hereafter, we denoted $\tilde{\bm{V}}^3$ by $\bm{V}_t\in\mathbb{R}^{N_t \times S_3}$ for simplicity, where the subscript $t$ denotes the task index and $S_3$ is the total number of nodes at FC3. In contrast to the softmax layer, we evaluated the output weights by solving a linear equation system. Hence, the learning objective for IF2Net is the following regularized least-squares minimization problem:

\begin{equation}\label{LSM}
\arg\min_{\bm{\beta}_t}: \Vert\bm{V}_t\bm{\beta}_t-\bm{Y}_t\Vert^2_F+\mu\Vert\bm{\beta}_t\Vert^2_F
\end{equation}

\noindent where $\Vert\cdot\Vert_F$ is the Frobenius norm, $\bm{\beta}_t$ is the output weight matrix, $\bm{Y}_t$ is the training sample label, and $\mu$ is the trade-off coefficient. Thus, the corresponding network prediction $\hat{\bm{Y}}_t=\bm{V}_t\bm{\beta}_t$ can be made at the final output layer by projecting the output weight updates $\Delta\bm{\beta}_t$ into the parsimonious orthogonal space, where the update direction is guided by an orthogonal projection matrix $\bm{P}_t$ (see Fig. \ref{Fig_ODM}). When learning a new task, old output weights are updated in the orthogonal direction of learned representations to avoid infringing upon the previously acquired information, even after experiencing a sequence of tasks. Specially,

\begin{equation}\label{OBP}
\bm{\beta}_t=\bm{\beta}_{t-1}-\eta \bm{P}_t\Delta\bm{\beta}_t
\end{equation}

\noindent where $\eta$ is the learning rate. Hence, during the decision-making process, let nonzero matrix $\bm{P}_t$ satisfy $\bm{A}_t\bm{P}_t=\bm{O}$ $(t=1,2,\dots)$, in which $\bm{A}_t=[\bm{V}_1^\mathrm{T},\bm{V}_2^\mathrm{T},\dots,\bm{V}_{t-1}^\mathrm{T}]^\mathrm{T}$ consists of all the previously tweaked representations. The mathematical proof is as follows:

\begin{figure}[tbp]
\centering
\includegraphics[width=1.0\columnwidth]{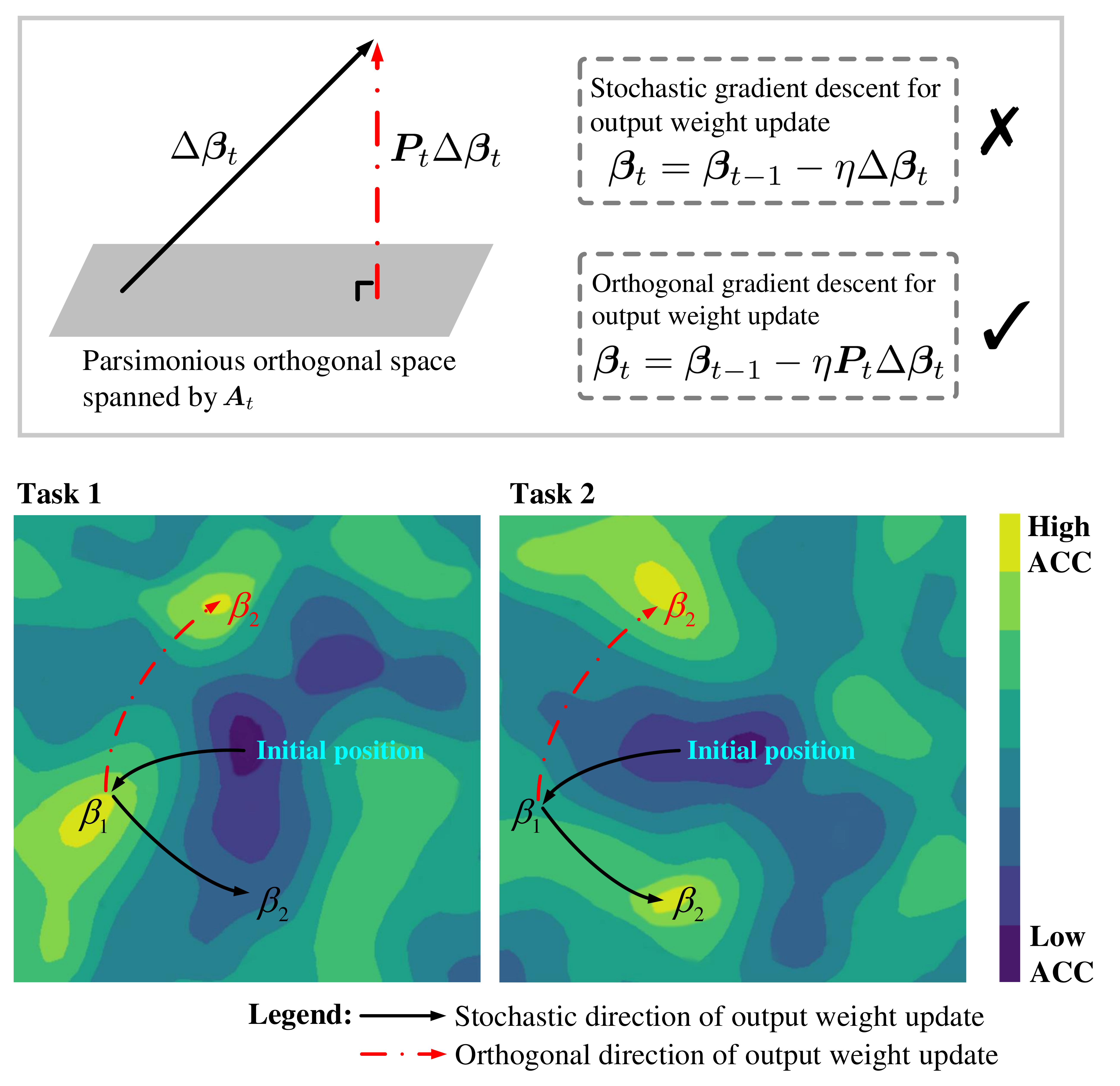}
\caption{Orthogonalization-based decision-making process. We first projected the output weight updates $\Delta\bm{\beta}_t$ of task $t$ into the parsimonious orthogonal space of all previous tasks, spanned by representations from the penultimate layer $\bm{A}_t$. Then, an orthogonal projection matrix $\bm{P}_t$ for modulating gradients guarantees output weights unaffected relative to all tasks seen. For example, the $\bm{\beta}_2$ updated by stochastic gradient descent failed to recognize task 1 while it would perform well on tasks seen by orthogonal gradient descent. Best viewed in color.}
\label{Fig_ODM}
\end{figure}

Suppose there have been $t$ ($t=1,2,\dots,T-1$) tasks seen so far, and when IF2Net is trained on $T$th task, its current prediction is $\hat{\bm{Y}}_T=\bm{V}_T\bm{\beta}_T$ in which $\bm{\beta}_T$ is updated with (\ref{OBP}) to maintain model plasticity on something new, and each $\bm{P}_t$ can be constructed in the following form:

\begin{equation}\label{P_form}
\begin{split}
	\bm{P}_t&=\bm{I}-\bm{A}_t^\mathrm{T}(\bm{A}_t\bm{A}_t^\mathrm{T}+\alpha\bm{I})^{-1}\bm{A}_t\\
	&=\alpha(\bm{A}_t^\mathrm{T}\bm{A}_t+\alpha\bm{I})^{-1}
\end{split}
\end{equation}

\noindent where $\alpha$ is a relatively small constant \cite{NMI2019OWM,li2021gopgan}. At the test time, the final model can predict $\hat{\bm{Y}}_{t}$ from any of the tasks 1 to $T-1$ without telling the task identifier, that is, 

\begin{equation}\label{proof}
\begin{split}
	\hat{\bm{Y}}_{t}=&\bm{V}_t\bm{\beta}_T\\
	=& \bm{V}_t(\bm{\beta}_{T-1}-\eta \bm{P}_T\Delta\bm{\beta}_T)\\
	=& \bm{V}_t(\bm{\beta}_t-\eta \sum_{k=t+1}^{T}\bm{P}_k\Delta\bm{\beta}_k)\\
	=& \bm{V}_t\bm{\beta}_t-\eta \sum_{k=t+1}^{T}\bm{V}_t\bm{P}_k\Delta\bm{\beta}_k\\
	=& \bm{V}_t\bm{\beta}_t
\end{split}
\end{equation}

\noindent Note that $\bm{A}_k\bm{P}_k=\bm{O}$ ($k=t+1,\dots,T$) and $\bm{V}_t$ belongs to $\bm{A}_k$. This completes the proof of the learning-without-forgetting decisions of IF2Net. $\hfill\square$

From the algorithm perspective, however, it is difficult to find an orthogonal projection matrix $\bm{P}_t$ that can always satisfy $\bm{A}_t\bm{P}_t=\bm{O}$ $(t=1,2,\dots)$. To make the computation of the projector more efficient, we used an iterative method to approximate the parsimonious orthogonal space spanned by $\bm{V}_t$, much akin to the recursive least squares (RLS) \cite{haykin2002adaptive,engel2004kernel}, which can be used to train feed-forward neural networks to achieve fast convergence. Thus, we regarded the first $k$ rows of $\bm{V}_t$ as a mini-batch; that is, $\bm{V}_t(k)=[\bm{v}_1,\bm{v}_2,\dots,\bm{v}_k]^\mathrm{T}$. Similarly, $\bm{V}_t(k+1)^\mathrm{T}\bm{V}_t(k+1)=\bm{V}_t(k)^\mathrm{T}\bm{V}_t(k)+\bm{v}_{k+1}^\mathrm{T}\bm{v}_{k+1}$. We denoted $\bm{P}_t(k)$ as the iterative form associated with $\bm{V}_t(k)$. Hence, (\ref{P_form}) can be expressed as:

\begin{equation}\label{P_k}
\begin{split}
	\bm{P}_t(k+1)&=\alpha(\bm{V}_t(k+1)^\mathrm{T}\bm{V}_t(k+1)+\alpha\bm{I})^{-1}\\
	&=\alpha(\bm{V}_t(k)^\mathrm{T}\bm{V}_t(k)+\bm{v}_{k+1}^\mathrm{T}\bm{v}_{k+1}+\alpha\bm{I})^{-1}
\end{split}
\end{equation}

\noindent Then, we applied the Woodbury matrix identity \cite{golub2013matrix} as a solution to (\ref{P_k}), as is presented below. Let $\bm{A}$ and $\bm{B}$ be two positive-definite $M\times M$ matrices related to 

\begin{equation}\label{Woodbury_A}
\bm{A}=\bm{B}^{-1}+\bm{C}\bm{D}^{-1}\bm{C}^\mathrm{T}
\end{equation}

\noindent where $\bm{D}$ is a positive-definite $N\times N$ matrix and $\bm{C}$ is an $M\times M$ matrix. Using the matrix inversion lemma, we obtained 

\begin{equation}\label{Woodbury_A_p}
\bm{A}^{-1}=\bm{B}-\bm{B}\bm{C}(\bm{D}+\bm{C}^\mathrm{T}\bm{B}\bm{C})^{-1}\bm{C}^\mathrm{T}\bm{B}
\end{equation}

\noindent By specifying $\bm{B}^{-1}=\bm{V}_t(k)^\mathrm{T}\bm{V}_t(k)+\alpha\bm{I}$ in (\ref{P_k}), we have $\bm{B}^{-1}=\alpha\bm{P}_t(k)^{-1}$. Furthermore, let $\bm{C}=\bm{v}_{k+1}^\mathrm{T}$ and $\bm{D}=\bm{I}$. According to (\ref{Woodbury_A})-(\ref{Woodbury_A_p}), we have

\begin{equation}\label{P_kk}
\bm{P}_t(k+1)=\bm{P}_t(k)-\dfrac{\bm{P}_t(k)\bm{v}_{k+1}^\mathrm{T}\bm{v}_{k+1}\bm{P}_t(k)}{\alpha\bm{I}+\bm{v}_{k+1}\bm{P}_t(k)\bm{v}_{k+1}^\mathrm{T}}
\end{equation}

Computing the orthogonal projector using (\ref{P_kk}) has two main strengths. (i) Iterative implementation in an efficient online manner avoids the matrix-inverse operation defined in (\ref{P_form}), which significantly facilitates the update process. (ii) Each batch is treated as a different task and breaks the limitation of storing representations from all previous tasks, for example, $\bm{V}_1,\bm{V}_2,\dots$, because only the currently learned $\bm{V}_t$ and the most recently updated $\bm{P}_t(k)$ are needed; therefore,
further accelerates the processing. \textbf{Algorithm \ref{alg_2}} illustrates the decision-making process of IF2Net.  

\begin{remark}\label{Remark1}
We noted that taking the “remembering” into consideration, the above orthogonalization-based decision-making can absorb a new input-to-output mapping without interfering with the previously learned ones. However, the process may be less effective for “learning.” This is because orthogonal regularization updates the output weights along the constrained direction, and it is potentially unfavorable to find the optima during sequential training, especially for task 1, which functions as the pre-training or initialization operation. To alleviate this, we retrospect the optimization problem in (\ref{LSM}), and a closed-form solution for parameter initialization is given by: 

\begin{equation}\label{beta_N}
	\bm{\beta}_1^\ast=(\bm{V}_1^\mathrm{T}\bm{V}_1+\mu\bm{I})^{-1}\bm{V}_1^\mathrm{T}\bm{Y}_1, N_1\geq S_3
\end{equation}

Alternatively, 

\begin{equation}\label{beta_L}
	\bm{\beta}_1^\ast=\bm{V}_1^\mathrm{T}(\bm{V}_1\bm{V}_1^\mathrm{T}+\mu\bm{I})^{-1}\bm{Y}_1, N_1 < S_3
\end{equation}

The computational complexity introduced by the matrix inversion operation can be circumvented by using either (\ref{beta_N}) or (\ref{beta_L}), depending on the sample size $N_1$ or total representation dimension $S_3$. Empirically, we can further start with a selected mini-batch to compute the closed-form solution, followed in the same way as the subsequent updates of the output weight $\bm{\beta}_t$ $(t\geq 2)$. In Section \ref{Initialization}, we demonstrated the effectiveness of this analytical initialization for output weight. Therefore, given the constrained update direction resulting from layer-wise orthogonalization, we simplified it to one layer for the final decision-making and introduced an effective initialization manipulation, eliminating the inflexibility to the fullest extent.
\end{remark}

\begin{algorithm}[htb]
\caption{Orthogonalization based decision-making}
\label{alg_2}
\begin{algorithmic}[1]
	\Require 
	Datasets $\bm{D}_t=\{(\bm{X}_t,\bm{Y}_t)\}$ of task $t$; 
	A classifier with a single-head output layer;
	The positive scalars $\mu$, $\eta$, and $\alpha$;
	Rrepresentations of penultimate layer $\bm{V}_t$.
	
	\Ensure 
	Network predictions on tasks seen so far $\hat{\bm{Y}}_t$.
	
	\State \textit{\# decision-making on sequential tasks}
	\For{$t=1, 2,\dots,T$}
	\If{$t=1$}
	\State Initialize $\bm{\beta}_1^\ast$ based on (\ref{beta_N}) or (\ref{beta_L});
	\State Iteratively approximate $\bm{P}_t$ based on (\ref{P_kk});	
	\Else	
	\State \textit{\# updates in the parsimonious orthogonal space}
	\State Orthogonalize gradient by $\bm{P}_t\Delta\bm{\beta}_t$;
	\State Update $\bm{\beta}_t$ based on (\ref{OBP});
	\State Iteratively approximate $\bm{P}_t$ based on (\ref{P_kk});
	
	\EndIf
	\EndFor
\end{algorithmic}
\end{algorithm}

\subsection{Integrating the Strengths of External Conditions}
To highlight the generality of our approach, we further explored IF2Net's compatibility with prior work and leveraged them as external conditions, such as replay-, regularization-, and dynamic architecture-based strategies, achieving a marginal boost in learning performance. Given space limitations, we considered the typical EWC algorithm as an example and rename it IF2Net-EWC. Both IF2Net and IF2Net-EWC adopted the same representation-level learning in the hidden layers but slightly different policies for updating the output weights during sequential training.

Especially, suppose there have been $T-1$ $(T\geq 2)$ tasks so far. When the $T$th task is presented, the improved learning objective for IF2Net-EWC can be represented as:

\begin{equation}\label{LSM_EWC}
	\begin{split}
		\arg\min_{\bm{\beta}_T}\!&: \frac{1}{2N_T}\!\Vert\bm{V}_T\bm{\beta}_T\!-\!\bm{Y}_T\Vert^2_F\!+\!\frac{\mu}{2}\!\sum_{t=1}^{T-1}\!\Vert\bm{Q}_t\odot(\bm{\beta}_T\!-\!\bm{\beta}_{T-1})\Vert^2_F \\
		s.t.\!&: \bm{\beta}_T=\bm{\beta}_{T-1}\!-\!\eta \bm{P}_T\Delta\bm{\beta}_T
	\end{split}
\end{equation}

\noindent where $N_T$ is the number of current training samples, $\odot$ is the element-wise product, $\bm{\beta}_{T-1}$ is the most recently learned output weight, and $\bm{\mathcal{F}}_t=\bm{Q}_t\odot\bm{Q}_t$ is the Fisher information matrix that indicates the parameter importance, as is formulated below. 

\begin{equation}\label{myEFIM}
	\bm{\mathcal{F}}_t=\frac{1}{N_t}\sum_{p=1}^{N_t}\nabla_{\bm{\bm{\beta}}_T}\log p(\bm{\bm{\beta}}_T|\bm{D}_t)\nabla_{\bm{\bm{\beta}}_T}\log p(\bm{\bm{\beta}}_T|\bm{D}_t)^\mathrm{T}
\end{equation}

\noindent Thus, IF2Net-EWC retains the knowledge of old tasks by constraining the output weights to stay in a low-error region while learning a new task, which only allows the new task to change its output weights that are not important for old tasks in the parsimonious orthogonal space. In other words, each output weight with small importance does not significantly affect the performance and can, therefore, be changed along the orthogonal direction to minimize the objective function, while, ideally, it should be left unchanged with large importance. Moreover, $\bm{D}_t$ no longer needs to be revisited or stored once $\bm{\mathcal{F}}_t$ is obtained. In summary, IF2Net-EWC further incorporates the update quota according to the parameter importance to implement selective weight consolidation compared to IF2Net.

\begin{remark}\label{Remark2}
	Intuitively, only a single term should be maintained and anchored at the output weights associated with the latest task $T-1$ because the most recently learned weights $\bm{\beta}_{T-1}$ inherit the previous weights $\bm{\beta}_t$ $(t=1,2,\dots,T-2)$. Alternatively, multiple penalty terms can be incorporated into (\ref{LSM_EWC}) by replacing $\bm{\beta}_{T-1}$ with $\bm{\beta}_t$. The latter forces a model to remember older tasks more vividly by double counting the data from previous tasks, which might compensate for the fact that older tasks are more difficult to remember \cite{huszar2017quadratic}. As a price, the algorithm exhibits linear growth in memory requirements as the number of tasks increases. Hence, we combined only IF2Net with a single term, as formulated in (\ref{LSM_EWC}).
\end{remark}

\subsection{Discussion} \label{Sec_Discussion}
\subsubsection{Can we directly use NN-RW for CL?} \label{Discussion_1}
As described in the related work section, NN-RW randomly assigns input weights (and biases), and only the output weights need to be tuned during the single-task learning process. Although some useful strategies, such as the supervisory mechanism in SCN \cite{wang2017stochastic}, and the \textit{lasso} algorithm in BLS \cite{TNNLS2017BLS}, have been introduced to better leverage the randomization nature, NN-RW and its variants are data-dependent and mainly perform well on the IID condition, making it unfeasible in the CL context. Specifically, the randomized learning technique can be used to extract discriminative features for a given task, but it fails to work in a sequence of tasks. This is because the discriminative information learned for a new task may not be sufficiently discriminative between old tasks and between old and new classes \cite{AAAI2021PCL}. More importantly, NN-RW with task-specific output weights limits its application to the single-task learning process, and no work has been presented to guide the updates of output weights to address forgetting. Therefore, further improvements are needed to extend NN-RW to the CL context, and our work fills this gap to a certain degree.

\subsubsection{Relationship between IF2Net and orthogonal gradient decent methods}  \label{Discussion_2}
We now discuss the relationship between the proposed approach, IF2Net, and the representative orthogonal gradient descent method, OWM \cite{NMI2019OWM}. Both serve the same purpose of defying catastrophic forgetting within neural networks. However, there are three main differences between IF2Net and OWM, which yield substantial performance distinctions. (i) Rather than orthogonal regularization updates of all the network weights layer-by-layer, we proposed leveraging randomization-based representation-level learning in several hidden layers, as discussed in Section \ref{Sec_RRL}. It is theoretically analyzed that the proposed technique can make the randomized responses of different tasks converge to their separate optima. (ii) We then applied orthogonalization to the final output layer for decision-making because the feature map extracted from deeper layers is more likely to contain task-specific information, and the deeper layer can easily forget previous knowledge \cite{tang2022learning}. Particularly, we projected the output weight updates into the parsimonious orthogonal space spanned by the obtained response from the preceding layer, that is, our method only needs to maintain one orthogonal projection matrix, while the number of projectors in OWM is equal to that of the network layers. Our method, together with a closed-form solution for parameter initialization, can facilitate the convergence process and eliminate inflexibility in original layer-wise practice. (iii) Our approach was compatible with most, if not all, CL methods because of its generality, meaning that one can easily combine IF2Net with the respective techniques in the existing CL methods to achieve new state-of-the-art results.

\subsubsection{Comparison through the lens of generalization bound} \label{Discussion_3}
We also discussed the generalization performance to understand why our method is better than OWM. Based on Rademacher complexity \cite{kakade2008complexity}, two supportive lemmas are presented to explain our intuitive argument, that is, orthogonalization-based decision-making in the final output layer is more prone to result in a lower generalization bound error during sequential training. 

This section focused on Rademacher complexity, which is a standard tool for binding the generalization error (and hence the sample complexity) of given classes of predictors \cite{bartlett2002rademacher,Shalev2014Understanding}. Formally,
given a real-valued function class $\mathcal{H}$ and set of data points $x_1,\dots,x_m \in \bm{X}$, we defined the (empirical) Rademacher complexity $\mathcal{\hat{R}}_m(\mathcal{H})$ as:
\begin{equation}\label{Rademacher_def}
	\mathcal{\hat{R}}_m(\mathcal{H}) = \mathbb{E}_{\bm{\epsilon}} \bigg[\sup_{h\in\mathcal{H}}\frac{1}{m}\sum_{i=1}^{m}h(x_i)\epsilon_i\bigg]
\end{equation}

\noindent where $\bm{\epsilon}=(\epsilon_1,\dots,\epsilon_m)$ is a vector uniformly distributed in \{-1,+1\} and $m$ is the sample size. Bartlett and Mendelson \cite{bartlett2002rademacher} provided the following generalization bound for the Lipschitz loss functions: 

\begin{lemma}\cite{bartlett2002rademacher}\label{Lemma_2}
	Assume the loss $\ell$ is Lipschitz (with respect to its first argument) with the Lipschitz constant $L_\ell$ and that $\ell$ is bounded by $c$. For any $\delta>0$ and with a probability of at least $1-\delta$ simultaneously for all $h\in\mathcal{H}$, we have
	
	\begin{equation*}
		\mathcal{L}(h)\leq \hat{\mathcal{L}}(h)	+ 2L_\ell \mathcal{\hat{R}}_m(\mathcal{H}) + c\sqrt{\frac{\log(1/\delta)}{2m}}
	\end{equation*}
	where $\mathcal{L}(h)=\mathbb{E}[\ell (h(x),y)]$ is the expected loss of $h$ and $\hat{\mathcal{L}}(h)=\frac{1}{m}\sum_{i=1}^{m}\ell(h(x_i),y_i)$ is the empirical loss. 	
\end{lemma}

We can introduce a bound on the Rademacher complexity term using \textbf{\textit{Lemma \ref{Lemma_2}}}.

\begin{lemma}\cite{golowich2018size}\label{Lemma_3}
	Let $\mathcal{H}_d$ be the class of real-valued networks of depth $d$ over the domain $\bm{X}=\{x:\Vert x\Vert\leq B\}$ in Euclidean space, where each parameter matrix $\bm{W}_j$ has a Frobenius norm at most $M_F(j)$, and with activation functions satisfying \textit{Lemma 1} in \cite{golowich2018size}. Then,
	
	\begin{equation*}
		\begin{split}
			\mathcal{\hat{R}}_m(\mathcal{H}_d) &\leq \frac{1}{m}\prod_{j=1}^{d}M_F(j)\cdot\bigg(\sqrt{2\log(2)d}+1\bigg)\sqrt{\sum_{i=1}^{m}\Vert x_i\Vert^2}\\
			&\leq \frac{B\bigg(\sqrt{2\log(2)d}+1\bigg)\prod_{j=1}^{d}M_F(j)}{\sqrt{m}}.
		\end{split}
	\end{equation*}
	where $\bm{W}_1,\dots,\bm{W}_d$ in each of the $d$ layers have the Frobenius norm $\Vert\cdot\Vert_F$ upper bounded by $M_F(1),\dots,M_F(d)$, respectively. 
\end{lemma}

Let us consider these lemmas in the continual learning context. Suppose that $t$ ($t=1,2,\dots,T-1$) tasks have been trained. During sequential training, one can ensure that the empirical loss $\hat{\mathcal{L}}(h)$ on the right-hand side of \textbf{\textit{Lemma \ref{Lemma_2}}} is unchanged by the orthogonal decent gradient method. However, the Rademacher complexity term $\mathcal{\hat{R}}_m(\mathcal{H}_d)$ can still increase for task $t$ when we update network weights at task $T$. Thus, while networks do not forget what was learned, the generalization error for the previous tasks could potentially increase when weights are updated for later tasks. Moreover, we can use the bound of the Rademacher complexity term in \textbf{\textit{Lemma \ref{Lemma_3}}}. For our method, $d = 1$ and domain $\bm{X}$ refer to the last hidden layer, whereas for OWM, $d > 1$ and domain $\bm{X}$ refer to the input layer. As a result, this leads to a major difference in how much they increase the Rademacher complexity term for task $t < T - 1$ during training for task $T$. From \textbf{\textit{Lemma \ref{Lemma_3}}}, it is expected that the proposed IF2Net does not significantly increase the Rademacher complexity term compared to OWM. We validated this hypothesis by plotting the value of the Rademacher complexity term for our method and OWM in Section \ref{Rademacher_complexity}.

\section{Experiments}
In this section, we presented extensive experiments to validate the superiority of the proposed approach using three evaluation metrics, six standard benchmark datasets, and 17 baseline methods in the CIL scenario. First, we introduced the experimental setting. We then provided the experimental results and discussion, following which we empirically analyzed the effectiveness of the core designs in our algorithm.

\subsection{Experimental Setting}
\textbf{Protocols.} 
For a fair comparison, we carefully selected the compared methods in the same environment. (i) Only samples of current task $t$ are available, except for replay-based methods, while test samples may come from any of tasks 1 to $t$ at test time, during which task identity is unknown. (ii) Without using a fixed task sequence, we ran each benchmark multiple times with randomly shuffled task orderings and then reported the means and/or standard deviations of these results. Hence, repeated runs will enter tasks with agnostic orders prior to training, which is more practical in an open-ended environment. (iii) We gave preference to the most well-recognized and best-performing baseline methods in the “single-head” setting and referred to the original codebases for implementation and hyperparameter selection to ensure the best possible performance.

\textbf{Datasets.}
We followed the common practice in the CIL scenario to simulate changing input distributions and emerging new classes by splitting benchmark datasets \cite{tang2022learning,AAAI2021PCL,van2020replay,wang2022learning,van2019three,hsu2018re,CVPR2021EFT}, including MNIST, FashionMNIST, Stanford Dogs, CUB-200, CIFAR-100, and ImageNet-Subset. For convenience, we used the nomenclature [DATASET]-$C/T$ to denote a task sequence with $C$ classes evenly divided into $T$ tasks in the CIL scenario; for example, the suffix indicates that a model needs to recognize $C/T$ new classes in each task. 


\textbf{Compared methods}.
We compared our approach with both classic and the latest CL baselines, which covered GEM \cite{NIPS2017GEM}, LOGD \cite{CVPR2021LOGD}, IL2M \cite{ICCV2019IL2M}, FS-DGPM \cite{deng2021flattening}, ARI \cite{wang2022anti} from replay-based approaches, EWC \cite{PNAS2017EWC}, SI  \cite{ICML2017SI}, MAS \cite{ECCV2018MAS}, OL-EWC \cite{ICML2018Online_EWC}, BiC \cite{wu2019large}, OWM \cite{NMI2019OWM} from regularization-based approaches, and RPS-Net \cite{NIPS2019RPS-Net}, EFT \cite{CVPR2021EFT}, PCL \cite{AAAI2021PCL} from dynamic architecture-based approaches. Additionally, we compared it with a representative NN-RW \cite{TNNLS2017BLS}, a naive approach of simply fine-tuning each new task (\textit{None}; can be seen as lower bound), and a network that is always trained using the data of all tasks so far (\textit{Joint}; can be seen as upper bound).

\begin{table*}[htbp]
	\renewcommand{\arraystretch}{1.5}  
	\caption{Performance comparisons on MNIST-10/5 and FashionMNIST-10/5 both measured by three evaluation metrics. All the results are (re) produced under 5 runs, with the mean and standard deviation reported. We also mark the best results in \textbf{Bold} and second-best results in \textbf{\textit{Italic}}.}
	\label{Table_MNIST_FMNIST}
	\centering
	\begin{tabular}{lccclccc}
		\toprule
		\multirow{3}{*}{\multirow{2}{*}{Method}} & \multicolumn{3}{c}{MNIST-10/5} &  & \multicolumn{3}{c}{FashionMNIST-10/5} \\ \cmidrule(r){2-8}
		& \multicolumn{3}{c}{Metric}            &  & \multicolumn{3}{c}{Metric}      \\ \cmidrule(r){2-4} \cmidrule(r){6-8}
		& ACC & BWT & FWT & & ACC & BWT & FWT \\ \midrule
		\textit{None} (lower bound)   &$\sim$19.91  &-  &-  &  &$\sim$19.81  &-  &-  \\
		\textit{Joint}~ (upper bound) &$\sim$98.56  &-  &-   &  &$\sim$96.61  &-  &-  \\ \cmidrule(r){1-8}
		NN-RW \cite{TNNLS2017BLS} &19.92$\pm$0.23 &$-$0.9926$\pm$0.0029 &\textbf{0.0002$\pm$0.0004}   &  &19.88$\pm$0.26 &$-$0.9921$\pm$0.0028 &\textbf{0.0003$\pm$0.0004}  \\ 
		EWC \cite{PNAS2017EWC} &38.68$\pm$3.44 &$-$0.6958$\pm$0.0272 &$-$0.0442$\pm$0.0141  & &37.97$\pm$7.58 &$-$0.6864$\pm$0.0247 &$-$0.0547$\pm$0.0731 \\
		MAS \cite{ECCV2018MAS} &44.61$\pm$6.62 &$-$0.0589$\pm$0.0674 &$-$0.5873$\pm$0.0651  & &34.91$\pm$5.47 &$-$0.6095$\pm$0.0859 &$-$0.1950$\pm$0.0298 \\
		OL-EWC \cite{ICML2018Online_EWC} &57.38$\pm$4.04 &$-$0.3773$\pm$0.0584 &$-$0.1128$\pm$0.0451  & &54.09$\pm$4.03 &$-$0.4014$\pm$0.0760 &$-$0.1376$\pm$0.0499 \\
		SI \cite{ICML2017SI}   &69.44$\pm$4.37 &$-$0.0436$\pm$0.0900 &$-$0.2943$\pm$0.0935  & &52.11$\pm$2.22 &$-$0.4871$\pm$0.0609 &$-$0.0824$\pm$0.0420 \\
		EFT \cite{CVPR2021EFT} &82.53$\pm$1.15 &$-$0.0856$\pm$0.0685 &$-$0.0953$\pm$0.0428  & &74.79$\pm$1.23 &$-$0.1328$\pm$0.0537 &$-$0.0875$\pm$0.0227\\
		OWM \cite{NMI2019OWM}  &87.16$\pm$0.65 &$-$0.0561$\pm$0.0138 &$-$0.0901$\pm$0.0171  &  &80.32$\pm$0.95 &$-$0.1434$\pm$0.0268 &$-$0.0953$\pm$0.0206\\
		GEM \cite{NIPS2017GEM} &94.10$\pm$0.37 &$-$0.0323$\pm$0.0054 &\textbf{\textit{$-$0.0154$\pm$0.0024}}  &  &81.95$\pm$1.86 &$-$0.0881$\pm$0.0258 &$-$0.0811$\pm$0.0252\\
		FS-DGPM \cite{deng2021flattening} &89.12$\pm$1.14 &$-$0.0841$\pm$0.0123 &$-$0.0403$\pm$0.0070  & &80.89$\pm$0.74 &$-$0.1177$\pm$0.0182 &$-$0.1073$\pm$0.0154\\
		BiC \cite{wu2019large} &93.93$\pm$0.58 &$-$0.0390$\pm$0.0069 &$-$0.0285$\pm$0.0044  & &82.36$\pm$0.72 &$-$0.1084$\pm$0.0309 &$-$0.0847$\pm$0.0243\\
		ARI \cite{wang2022anti} &93.60$\pm$0.57 &$-$0.0418$\pm$0.0013 &$-$0.0283$\pm$0.0024  & &82.89$\pm$0.83 &$-$0.0966$\pm$0.0104 &$-$0.1056$\pm$0.0091\\ 
		PCL \cite{AAAI2021PCL} &94.14$\pm$0.67 &$-$0.0314$\pm$0.0256 &$-$0.0289$\pm$0.0109  & &83.27$\pm$0.81&$-$0.1238$\pm$0.0135&\textbf{\textit{$-$0.0257$\pm$0.0108}}\\
		RPS-Net \cite{NIPS2019RPS-Net}&94.53$\pm$1.92 &\textbf{\textit{$-$0.0237$\pm$0.0102}} &$-$0.0302$\pm$0.0087  & &84.18$\pm$1.60&\textbf{\textit{$-$0.0285$\pm$0.0124}}&$-$0.0417$\pm$0.0085\\
		LOGD \cite{CVPR2021LOGD}&94.87$\pm$0.59 &$-$0.0393$\pm$0.0068 &$-$0.0718$\pm$0.0233  & &\textbf{\textit{84.39$\pm$0.47}} &$-$0.0918$\pm$0.0315 &$-$0.0615$\pm$0.0067\\
		IL2M \cite{ICCV2019IL2M}&\textbf{\textit{95.51$\pm$0.42}} &$-$0.0418$\pm$0.0046 &$-$0.0314$\pm$0.0323  & &82.38$\pm$2.04 &$-$0.1483$\pm$0.03452 &$-$0.0451$\pm$0.0192\\ \cmidrule(r){1-8}
		IF2Net  &\textbf{96.16$\pm$0.54} &\textbf{$-$0.0044$\pm$0.0008} &$-$0.0231$\pm$0.0059  & &\textbf{95.09$\pm$0.47} &\textbf{$-$0.0189$\pm$0.0074} &$-$0.0267$\pm$0.0112\\ \hdashline
		IF2Net-EWC &96.21$\pm$0.71 & $-$0.0103$\pm$0.0074  & $-$0.0272$\pm$0.0041 & & 95.01$\pm$0.61  & $-$0.0125$\pm$0.0061  & $-$0.0329$\pm$0.0146 \\ \bottomrule
	\end{tabular}
\end{table*}

\textbf{Architectures.}
In our experiments, all methods used similar-sized neural network architectures. For MNIST and FashionMNIST, a fully connected network with three hidden layers was used. No pre-trained CNN is used for toy tasks, as a simple model already generated very good results. For Stanford Dogs, CUB-200, CIFAR-100, and ImageNet-Subset, a standard ResNet (e.g., ResNet-18, ResNet-56) was employed to provide well-extracted features, similar to \cite{NMI2019OWM,hayes2020remind,AAAI2021PCL,wang2022learning,wu2022class} where a feature extractor was pre-trained by the selected partial categories. Then, the remaining classes for which the feature extractor was not trained are sequentially fed into a simple model to learn different mappings. For example, it took 100 classes from CUB-200 to pre-train with ResNet-56 and then start consecutively learning another 100 classes that the pre-trained model has not encountered. This is in line with the characteristics of human learning that prepares sufficiently for challenging tasks, instead of having no any prior knowledge. A pre-trained CNN is equally used in our method and all the baselines.

\textbf{Hyper-parameter.} For all baselines, we employed the open-source code released by their authors or a popular third-party code. Particularly, we used the SGD optimizer with an initial learning rate ($\eta$) of 0.01 and a batch size of 64 in our experiments. For the replay-based methods, we kept a random exemplar set of 4.4k for MNIST, FashionMNIST, CUB-200, and ImageNet-Subset, and restricted the exemplar memory budget to 2k samples for Stanford Dogs and CIFAR-100 by following the setting in \cite{NIPS2017GEM,CVPR2021LOGD}. Dynamic structure-based methods will have the same magnitude as a base network after learning all tasks or with no limits on their expansion size. For regularization-based methods, the trade-off is from set \{100, 1000, 10000, 100000\}. Note that the other hyperparameters are with reference to the original settings by default. In our method, the hyperparameter settings included $\lambda=0.01$ and $\gamma=0.4$ in (\ref{Theta_iteratrion}), $\mu=2^{-30}$ in (\ref{LSM}), and $\alpha=0.1$ in (\ref{P_kk}).

\subsection{Experimental Results}
\subsubsection{Experiments on toy examples}
\textbf{MNIST-10/5 and FashionMNIST-10/5.} 
Table \ref{Table_MNIST_FMNIST} lists the comparative results of MNIST-10/5 and FashionMNIST-10/5. Both the \textit{None} and NN-RW, made available for single-task learning, only remember the most recently learned task, and the previously trained ones have been thoroughly forgotten. In contrast, the proposed method achieved the highest ACC (96.16\%, 95.09\%), approaching \textit{Joint}, and BWT (-0.0044, -0.0189), approaching zero, proving that IF2Net is a forgetting-free method. The complementary FWT values of the GEM and PCL were better than those of IF2Net. However, their ACC values are lower than ours. We observed that the performance of EWC, MAS, OL-EWC, and SI is significantly inferior to the others, as the CIL scenario is particularly difficult for regularization-based methods. Most replay-based methods achieved marginally worse ACC than IF2Net, especially for IL2M and LOGD on the MNIST-10/5 task sequence.
Additionally, using randomly shuffled task ordering is more desirable in an open-ended environment, which can play a role in model fairness. In this way, the standard deviations reported that our method also has low-order sensitivity, with similar accuracies for each task sequence regardless of the random task orderings. Finally, the results of IF2Net-EWC showed a marginal boost in learning performance, indicating its generality and compatibility. 

\subsubsection{Experiments on challenging tasks}
\textbf{CUB-\{100/5, 100/10, 100/20\}.}  
Three more challenging task sequences split by the CUB were used, and the comparative results were plotted in a stacked bar chart (see Fig. \ref{Fig_CUB}), where a model must incrementally recognize fine-grained visual categorization tasks. Our method outperformed the selected state-of-the-art methods on ACC while being equipped with a competitive standard deviation. Particularly, the proposed IF2Net improved ACC with the second-best method IL2M by an absolute margin of 5.48\%, and the standard deviations of GEM, OWM, and BiC were slightly superior to ours, but their ACC was much inferior to IF2Net. Hence, the results on the CUB-\{100/5, 100/10, 100/20\} task sequences highlighted our method as a promising tool for alleviating forgetting. The standard deviation of PCL was high because it employed a multi-head setup where it branched into an exclusive output layer for the classes learned so far. Consequently, it is demanding that PCL correctly matched the corresponding heads for decision-making, given a sequence of fine-grained visual tasks.

\begin{figure}[htbp]
	\centering
	\includegraphics[width=1.0\columnwidth]{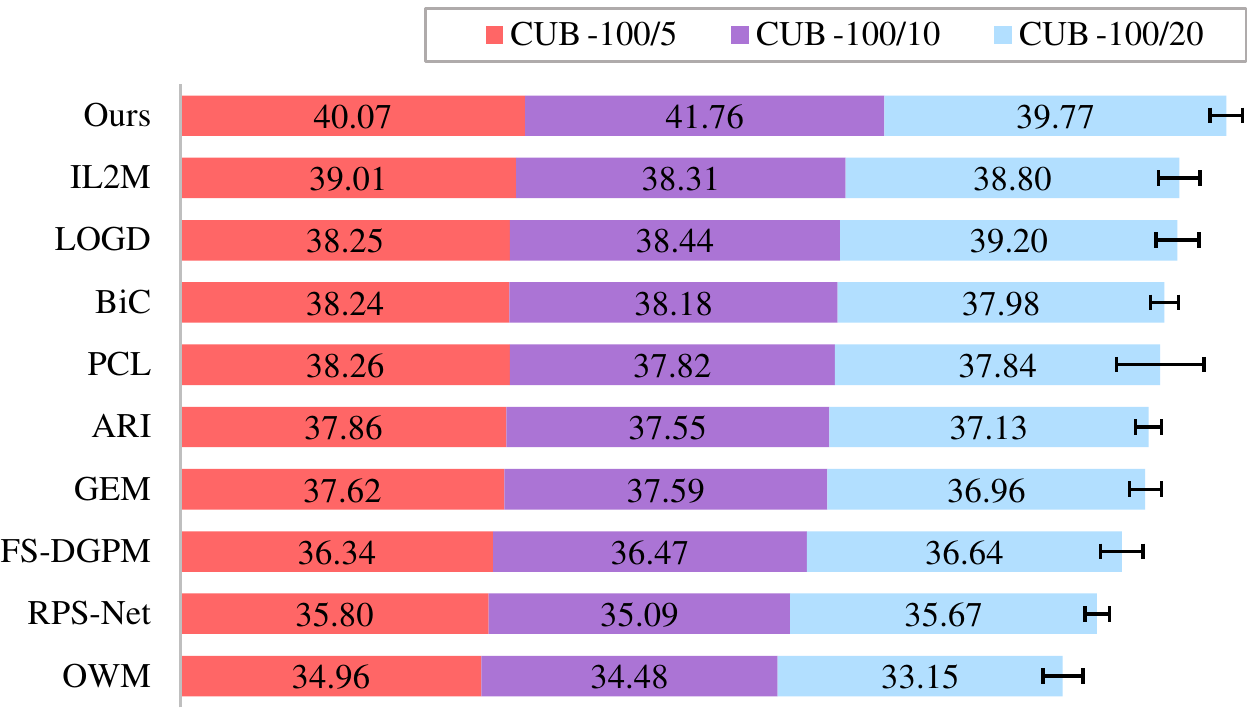}
	\caption{Test accuracy (\%) on CUB-\{100/5, 100/10, 100/20\} task sequences. All the results are (re) produced under 5 runs, with the separate mean and accumulative standard deviation reported.}
	\label{Fig_CUB}
\end{figure}

\textbf{Stanford Dogs-\{60/5, 60/10\}.}
Similarly, the results of the task sequences split by Stanford Dogs were consistent with those of the CUB, as depicted in Fig. \ref{Fig_StanfordDogs}. Note that the lower part of each sub-figure used the left vertical coordinate axis to report ACC, whereas the upper part used the right vertical coordinate axis to report the BWT and FWT results of the corresponding methods. Take Fig. \ref{StanfordDogs_5} as an example, and the proposed method exceeded the other baselines on ACC, among which IL2M is still the strongest baseline with 2.29\% lower than ours. Meanwhile, IF2Net achieved the best BWT value and second-best FWT value, which implies its strong ability to retain and transfer knowledge across fine-grained visual categorization tasks. This is also reflected in the comparison in Fig. \ref{StanfordDogs_10}.

\begin{figure}[tbp]
	\centering
	\subfloat[]{\includegraphics[width=0.90\columnwidth]{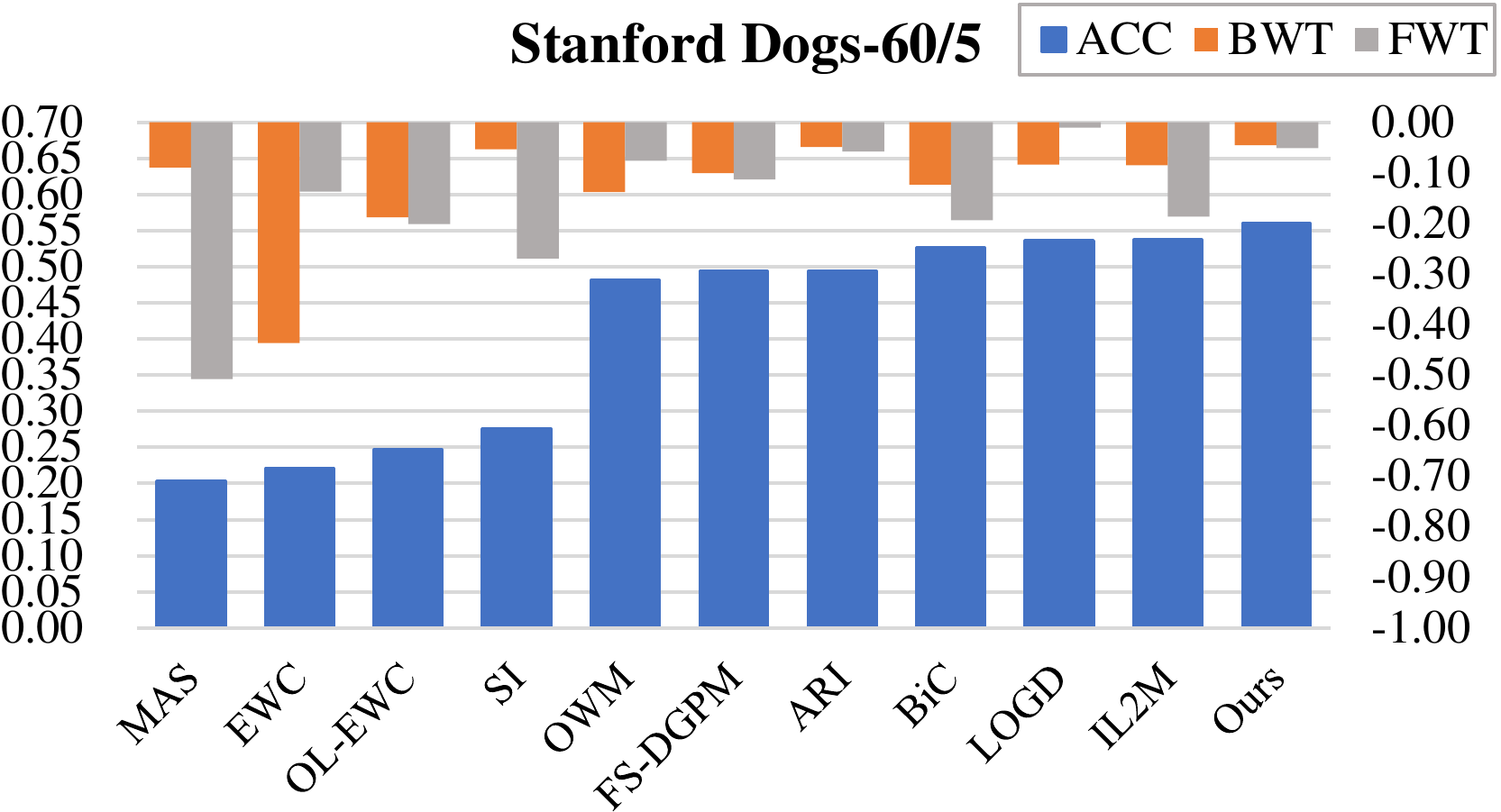}
		\label{StanfordDogs_5}}
	\hfil 
	\subfloat[]{\includegraphics[width=0.90\columnwidth]{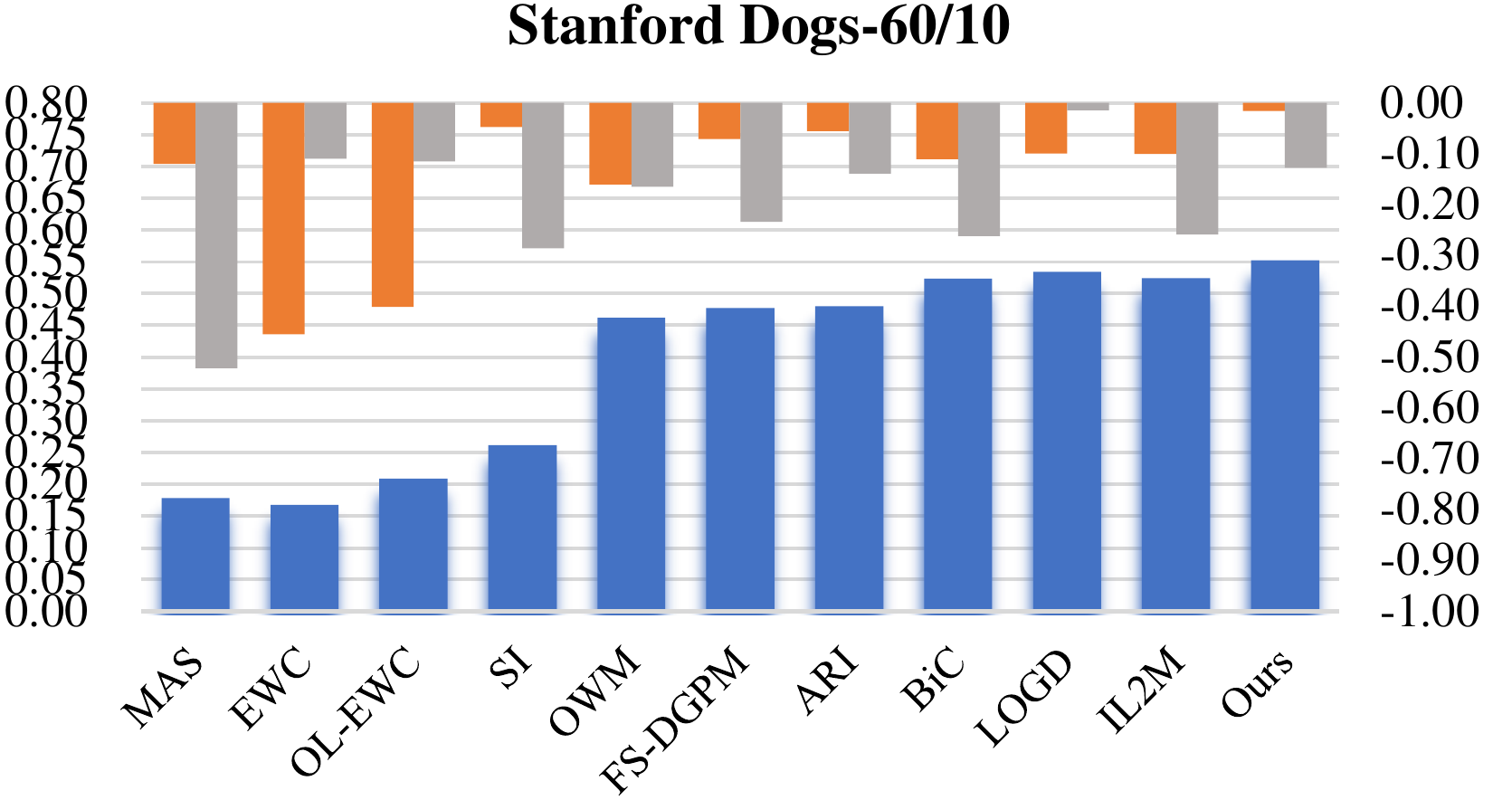}
		\label{StanfordDogs_10}}
	\caption{Performance comparisons on Stanford Dogs measured by three evaluation metrics. All the results are (re) produced under 5 runs. (a) Performance after sequentially learning five 12-class tasks. (b) Performance after sequentially learning ten 6-class tasks.}
	\label{Fig_StanfordDogs}
\end{figure}

\begin{table}[tbp]
	\renewcommand{\arraystretch}{1.5}  
	\setlength{\tabcolsep}{0.55mm} 
	\caption{Test accuracy (\%) on CIFAR-60/10 task sequence. We mark the best results in \textbf{Bold} and the second-best in \textbf{\textit{Italic}}. Note that the reported task $t$ accuracy is an average of all $1,2,\dots,t$ tasks seen so far.}
	\label{Table_CIFAR}
	\centering
	\begin{tabular}{lcccccccccc}
		\toprule
		\multirow{2}{*}{\textbf{Method}} & \multicolumn{10}{c}{\textbf{Task Number}}  \\ \cmidrule{2-11}
		& 
		\multicolumn{1}{c}{1} &
		\multicolumn{1}{c}{2} &
		\multicolumn{1}{c}{3} &
		\multicolumn{1}{c}{4} &
		\multicolumn{1}{c}{5} &
		\multicolumn{1}{c}{6} &
		\multicolumn{1}{c}{7} &
		\multicolumn{1}{c}{8} &
		\multicolumn{1}{c}{9} &
		\multicolumn{1}{c}{10} \\ \hline
		OWM     & \textbf{85.36} & 66.23 & 55.50 & 49.39 & 48.16 & 45.28 & 42.34 & 39.76 & 38.10 & 38.19 \\
		ARI     & 83.53 & 67.92 & 57.56 & 54.25 & 51.83 & 48.42 & 46.35 & 41.06 & 40.94 & 38.27 \\
		FS-DGPM & 83.45 & 74.08 & 64.44 & 59.66 & 53.13 & 48.80 & 47.28 & 45.06 & 41.83 & 39.42 \\
		RPS-Net & 83.83 & 66.91 & 65.05 & 64.25 & 58.16 & 52.83 & 49.04 & 47.47 & 45.31 & 43.01 \\
		LOGD    & 85.03 & 72.41 & 66.55 & 61.87 & 59.20 & 54.11 & 55.76 & \textit{\textbf{54.12}} & 51.53 & 49.00 \\
		GEM     & 84.96 & 74.53 & 65.22 & 60.96 & 60.07 & \textit{\textbf{59.14}} & \textit{\textbf{56.31}} & 53.35 & 52.96 & 49.83 \\
		IL2M    & \textit{\textbf{85.11}} & \textbf{76.44} & \textbf{69.84} & \textit{\textbf{65.42}} & \textit{\textbf{63.19}} & 58.63 & 54.35 & 54.10 & 50.87 & 51.78 \\
		BiC     & 84.60 & 72.13 & 67.64 & 64.80 & 60.12 & 58.42 & 56.10 & 54.00 & \textit{\textbf{53.77}} & \textit{\textbf{52.03}} \\  \hline
		Ours    & 84.71 & \textit{\textbf{74.58}} & \textit{\textbf{68.11}} & \textbf{65.45} & \textbf{63.70} & \textbf{59.94} & \textbf{59.38} & \textbf{59.63} & \textbf{57.22} & \textbf{56.98} \\ \bottomrule
	\end{tabular}
\end{table}

\textbf{CIFAR-\{60/10, 60/20\}.}
The performance of different CL methods on task sequences split by CIFAR-100 is reported in Table \ref{Table_CIFAR}, which focused on the changing trend of ACC with an increasing number of tasks. The results indicated the superior performance of the proposed method in the CIL scenario. Although it started with a minor difference in the first task, our method showed a relatively gentle ACC degradation during sequential training. Furthermore, compared with the second-best ACC from IL2M, our method achieved a relative gain of 4.95\% after the end of learning. In addition to Table \ref{Table_CIFAR}, Fig. \ref{CIFAR_20} vividly depicts the consistent results by incrementally learning three new classes per task.

\begin{figure}[tbp]
	\centering
	\includegraphics[width=1.0\columnwidth]{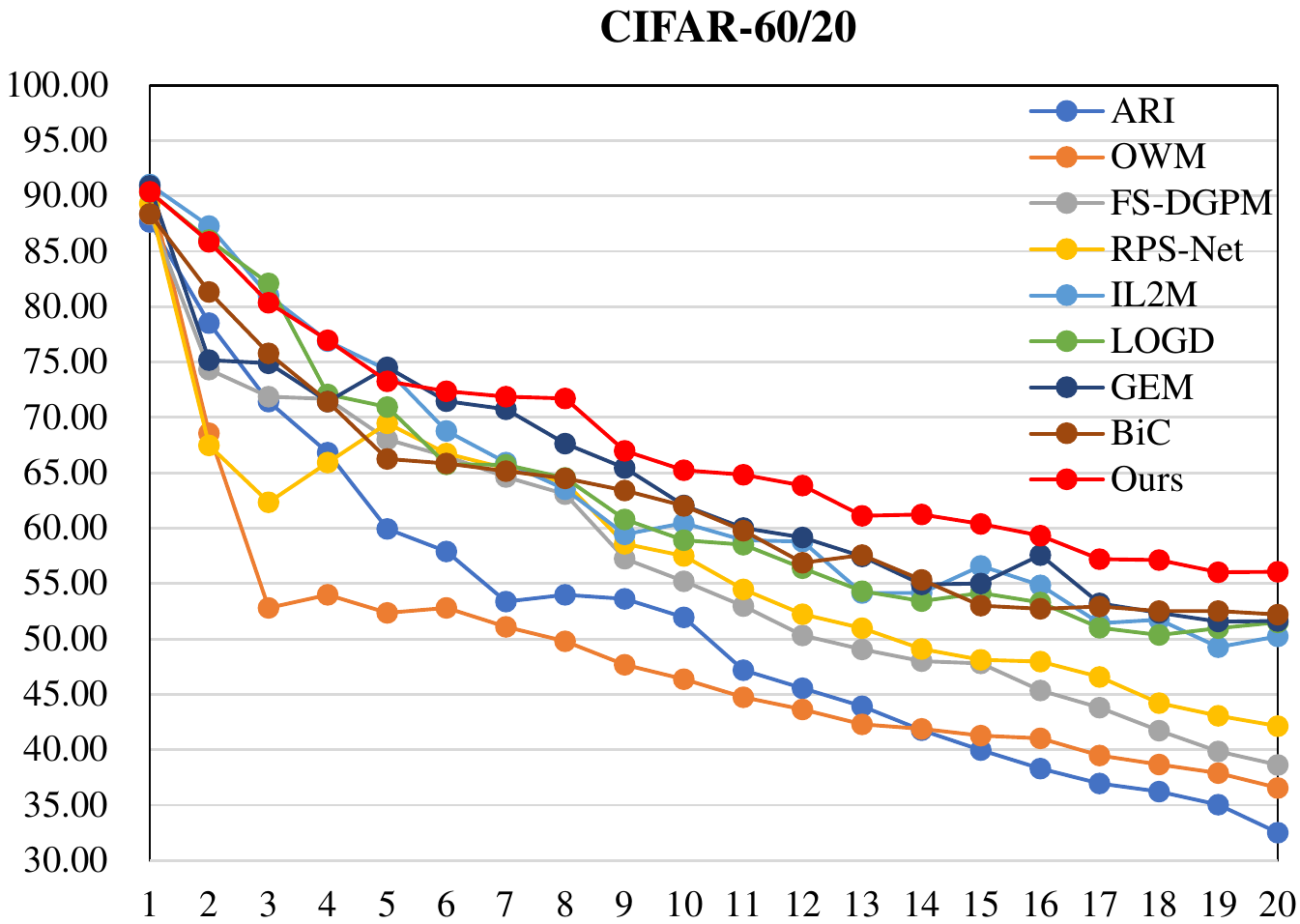}
	\caption{Test accuracy (\%) on CIFAR-60/20 task sequence, where each method needs to incrementally learn 3 new classes per task.}
	\label{CIFAR_20}
\end{figure}

\textbf{ImageNet-\{100/10, 100/25, 100/50\}.}
Table \ref{Table_ImangeNet} compares different methods on task sequences split by ImageNet-Subset, where the catastrophic forgetting problem becomes more challenging in either a tough individual task or a rather long task sequence. The results show that IF2Net surpasses the selected state-of-the-art methods with margins of 2.86\%, 2.08\%, and 1.53\% for ImageNet-\{100/10, 100/25, 100/50\}, respectively. Overall, extensive experiments validated the competitive performance of IF2Net against state-of-the-art methods, which are inherently immune to catastrophic forgetting by functionally maintaining the weights relative to each seen task untouched all the time.

\begin{table}[tbp]
	\renewcommand{\arraystretch}{1.5}  
	\setlength{\tabcolsep}{1.2mm} 
	\caption{Large-scale experiments on ImageNet in the CIL scenario that are measured by average test accuracy (\%). The methods are run under 5 random task orderings, with the mean and standard deviation reported.}
	\label{Table_ImangeNet}
	\centering
	\begin{tabular}{lccc}
		\toprule
		\multirow{2}{*}{\textbf{Method}} & \multicolumn{3}{c}{\textbf{Task Sequence}} \\ \cline{2-4}
		& ImageNet-100/10 & ImageNet-100/25 & ImageNet-100/50  \\ \hline
		FS-DGPM  &35.08$\pm$2.04   &34.96$\pm$1.96  &31.42$\pm$2.33    \\
		ARI      &36.99$\pm$1.17   &34.12$\pm$1.55  &30.68$\pm$2.42    \\ 
		RPS-Net  &40.25$\pm$2.65   &36.62$\pm$2.89  &29.33$\pm$3.45    \\ 
		GEM      &46.23$\pm$1.45   &44.62$\pm$1.13  &42.78$\pm$1.12    \\
		IL2M     &49.86$\pm$1.57   &48.05$\pm$1.25  &47.52$\pm$1.04    \\ \hline
		Ours     &52.72$\pm$1.65   &50.13$\pm$1.57  &49.05$\pm$2.75    \\
		\bottomrule	 
	\end{tabular}
\end{table}

\subsection{Ablation Study} \label{Ablation_Study}
Before concluding our work, we conducted ablation studies to provide a deeper understanding of our method. To this end, the empirical investigation of the effectiveness of core designs in IF2Net is detailed as follows.

\subsubsection{Advantage of output weight initialization using a closed-form solution} \label{Initialization}

\begin{table*}[htbp]
	\renewcommand{\arraystretch}{1.5}  
	\setlength{\tabcolsep}{0.8mm} 
	\caption{Empirical analysis of the output weight initialization using a closed-form solution on FashionMNIST-10/5. Under different learning rates ($\eta$) and epochs, the results were measured by the loss of learning task 1, the accuracy of learning task 1, and all. We marked the best results in \textbf{bold}.\\ Note that the number 1k, 2.5k, and 5k were the selected mini-batch for analytic initialization.}
	\label{Table_Param_Initialization}
	\centering
	\begin{tabular}{llllccclccclccc}
		\toprule
		\multirow{2}{*}{\textbf{Output Weights}} &  & \multicolumn{1}{l}{\multirow{2}{*}{$\eta$}} &  & \multicolumn{3}{c}{\textbf{Epoch 1}} &  & \multicolumn{3}{c}{\textbf{Epoch 5}} &  & \multicolumn{3}{c}{\textbf{Epoch 10}} \\ \cmidrule{5-7} \cmidrule{9-11} \cmidrule{13-15}
		&  & \multicolumn{1}{c}{} &  & Loss-task1 & ACC-task1 & ACC-All &  & Loss-task1 & ACC-task1 & ACC-All &  & Loss-task1 & ACC-task1 & ACC-All \\ \hline
		\multirow{2}{*}{Random Initialization} &  & 0.02 &  & 7.0902 & 0.9651 & 0.9022 &  & 2.5381 & 0.9750 & 0.9056 &  & 3.0367 & 0.9696 & 0.8689 \\
		&  & 0.002 &  & 6.6710 & 0.7933 & 0.6367 &  & 4.4575 & 0.7843 & 0.6553 &  & 2.0224 & 0.8975 & 0.6173 \\ \hline
		\multirow{2}{*}{Analytic Initialization (1k)} &  & 0.002 &  & 0.5013 & 0.9698 & 0.9281 &  & 0.2815 & 0.9713 & 0.9183 &  & 0.1369 & 0.9833 & 0.9189 \\ 
		&  & 0.0002 &  & 0.1444 & 0.9615 & 0.9237 &  & 0.0929 & 0.9717 & 0.9293 &  & 0.0634 & 0.9833 & 0.9200 \\ \cmidrule{3-15}
		Analytic Initialization (2.5k) &  & 0.0002 &  & 0.0923 & 0.9730 & 0.9465 &  & \textbf{0.0548} & \textbf{0.9848} & 0.9502 &  & \textbf{0.0489} & \textbf{ 0.9900} & 0.9411 \\ \cmidrule{3-15}
		Analytic Initialization (5k) &  & 0.0002 &  & \textbf{0.0910} & \textbf{0.9735} & \textbf{0.9505} &  & 0.0553 & 0.9810 & \textbf{0.9580} &  & 0.0550 & 0.9830 & \textbf{0.9482} \\ \bottomrule
	\end{tabular}
\end{table*}

As stated earlier, we used analytic initialization instead of random initialization for task 1; that is, we started with a selected mini-batch for computing the closed-form solution $\bm{\beta}_1^\ast$, followed in the same way as the subsequent updates of the output weight $\bm{\beta}_t$ $(t\geq 2)$. To investigate the benefits of this strategy, we compared two types of output weight initializations. Specifically, different $\eta$ and epochs were considered in the ablation study, including (i) randomly initialized output weights and (ii) starting with a selected mini-batch for analytically computing the output weights. Table \ref{Table_Param_Initialization} presents comparative results based on the FashionMNIST task sequence. We observed that with a proper $\eta$, analytic initialization can be converged more quickly than its random counterpart, as measured by the loss value, test accuracy for task 1, and ACC value for all. This means that it can effectively alleviate the restriction of the orthogonal gradient descent in the direction of the parameter update. Meanwhile, analytic initialization leveraging 2.5k (approximately 20\% of all available) samples from Task 1 can yield competitive results. By contrast, IF2Net with random initialization makes it difficult to find the optima even after ten epochs. 

\subsubsection{Efficiency of the node blocks in hidden layers} \label{Node_blocks}
We proposed randomization-based representation-level learning, in which several node blocks were randomly allocated to better exploit hidden information among the training samples. Taking FC1 as an example, Table \ref{Table_Node_blocks} illustrates the extent to which different combinations contributed to the learning performance. We observed that with the fixed 100 hidden nodes in FC1, IF2Net equipped with too small or too large node blocks would significantly degrade the values of ACC, BWT, and FWT. By contrast, some pairs (10, 10) and (25, 4) used in our experiments performed well, implying that appropriate node blocks are indispensable for the proposed representation-level learning process. Efficiency analysis provided some prior knowledge of node block selection, and it does not mean that the settings listed were the optimal choice. In other words, user-specified fine-tuning based on the recommended examples might yield better single and overall performance results.

\begin{table}[htbp]
	\renewcommand{\arraystretch}{1.5}  
	\caption{Empirical analysis of the node blocks on FashionMNIST-10/5. With different combinations, the results were measured by three evaluation metrics. The randomly allocated nodes in FC1 were composed of $n_1$ groups of node blocks, with each group $s_1$ nodes.}
	\label{Table_Node_blocks}
	\centering
	\begin{tabular}{lccccc}
		\toprule
		\multicolumn{2}{l}{\textbf{Node Block}} &  & \multicolumn{3}{c}{\textbf{Metric}} \\ \cmidrule{1-2} \cmidrule{4-6}
		$n_1$ & $s_1$ &  & ACC & BWT & FWT \\ \hline
		100 & 1 &  & 11.18$\pm$7.66 & 0.0002$\pm$0.0012 & -0.4751$\pm$0.0821 \\
		50 & 2 &  & 62.18$\pm$22.95 & -0.1115$\pm$0.0075 & -0.3191$\pm$0.1946 \\
		25 & 4 &  & 95.41$\pm$1.35 & -0.0072$\pm$0.0018 & -0.0316$\pm$0.0123 \\
		10 & 10 &  & 94.76$\pm$0.27 & -0.0140$\pm$0.0083 & -0.0357$\pm$0.0115 \\
		5 & 20 &  & 91.68$\pm$1.74 & -0.0342$\pm$0.0084 & -0.0554$\pm$0.0138 \\
		1 & 100 &  & 75.76$\pm$2.92 & -0.0845$\pm$0.0104 & -0.2022$\pm$0.0322 \\ \bottomrule		 
	\end{tabular}
\end{table}

\subsubsection{Impact of increasing orthogonalized layers} \label{Orthogonalized_layers}
One of the most apparent differences between IF2Net and OWM is that the latter only induces orthogonalization to the final output layer instead of layer-wise operations. To explore the effect of gradually adding orthogonalized layers, we denote the
corresponding cases: (a) Four Layers/OWM; (b) FC1-Three Layers; (c) FC1-FC2-Two layers; (d) FC1-FC2-FC3-One Layer/IF2Net with random initialization; (e) FC1-FC2-FC3-One Layer/IF2Net with analytic initialization. Table \ref{Table_orthogonaqlized_layers} compares different cases on FashionMNIST in which GPU memory usage is proportional to the number of orthogonalized layers. Furthermore, with only one or two orthogonalized layer(s), IF2Net performed well on the four metrics. Interestingly, the ACC produced from Case (d) is lower than that of Case (c) and Case (e), and replacing random initialization with its analytic implementation achieves counterattack results. Hence, we applied output-layer orthogonalization to the final decision-making process.

\begin{table}[htbp]
	\renewcommand{\arraystretch}{1.5}  
	\setlength{\tabcolsep}{2.0mm} 
	\caption{Empirical analysis of increasing several orthogonalized layers on FashionMNIST-10/5. In different cases, the results are measured by four evaluation metrics. Note that the metric GPU (MB) refers to memory usage during learning the last task.}
	\label{Table_orthogonaqlized_layers}
	\centering
	\begin{tabular}{lcccc}
		\toprule
		\multirow{2}{*}{\textbf{Case}} & \multicolumn{4}{c}{\textbf{Metric}}  \\ \cmidrule{2-5}
		& ACC & BWT & FWT & GPU \\ \hline
		(a) & 0.8058$\pm$0.0101 & -0.2216$\pm$0.0140 & -0.0197$\pm$0.0055 & 11635 \\
		(b) & 0.9214$\pm$0.0052 & -0.0212$\pm$0.0067 & -0.0125$\pm$0.0042 & 5281 \\
		(c) & 0.9428$\pm$0.0045 & -0.0158$\pm$0.0112 & -0.0165$\pm$0.0043 & 3435 \\
		(d) & 0.9090$\pm$0.0201 & -0.0292$\pm$0.0217 & -0.0758$\pm$0.0528 & 829 \\
		(e)  & 0.9526$\pm$0.0088 & -0.0071$\pm$0.0060 & -0.0289$\pm$0.0173 & 829 \\ \bottomrule
	\end{tabular}
\end{table}

\subsubsection{Superiority in generalization bound error} \label{Rademacher_complexity}
We followed the same setting in the previous section such that one can correspond to the Rademacher complexity for different cases. Specifically, we randomly took the value in \{-1,+1\} with equal probability as the training labels of each task and computed the Rademacher complexity based on (\ref{Rademacher_def}). Then, we accumulated the results of all previous tasks during each training session, measuring how the function class (current model) incrementally fit the randomly sampled labels. Fig. \ref{Fig_Rademacher_complexity} plotted the value of Rademacher complexity term for our method and OWM over the whole training sessions. We observed that the generalization bound error increases with the number of new tasks, as stated in \textbf{\textit{Lemma \ref{Lemma_2}}}. Especially, the IF2Net with analytic initialization corresponding to Case (e) used one orthogonalized layer in the final output layer and obtained the best generalization bound error. Contrastingly, for example, the more orthogonalized the layers, the faster the error increases, as stated in \textbf{\textit{Lemma \ref{Lemma_3}}}. Hence, this ablation study distinguished why our method is better than OWM and proved its superiority in generalization bound error.

\begin{figure}[htbp]
	\centering
	\includegraphics[width=1.0\columnwidth]{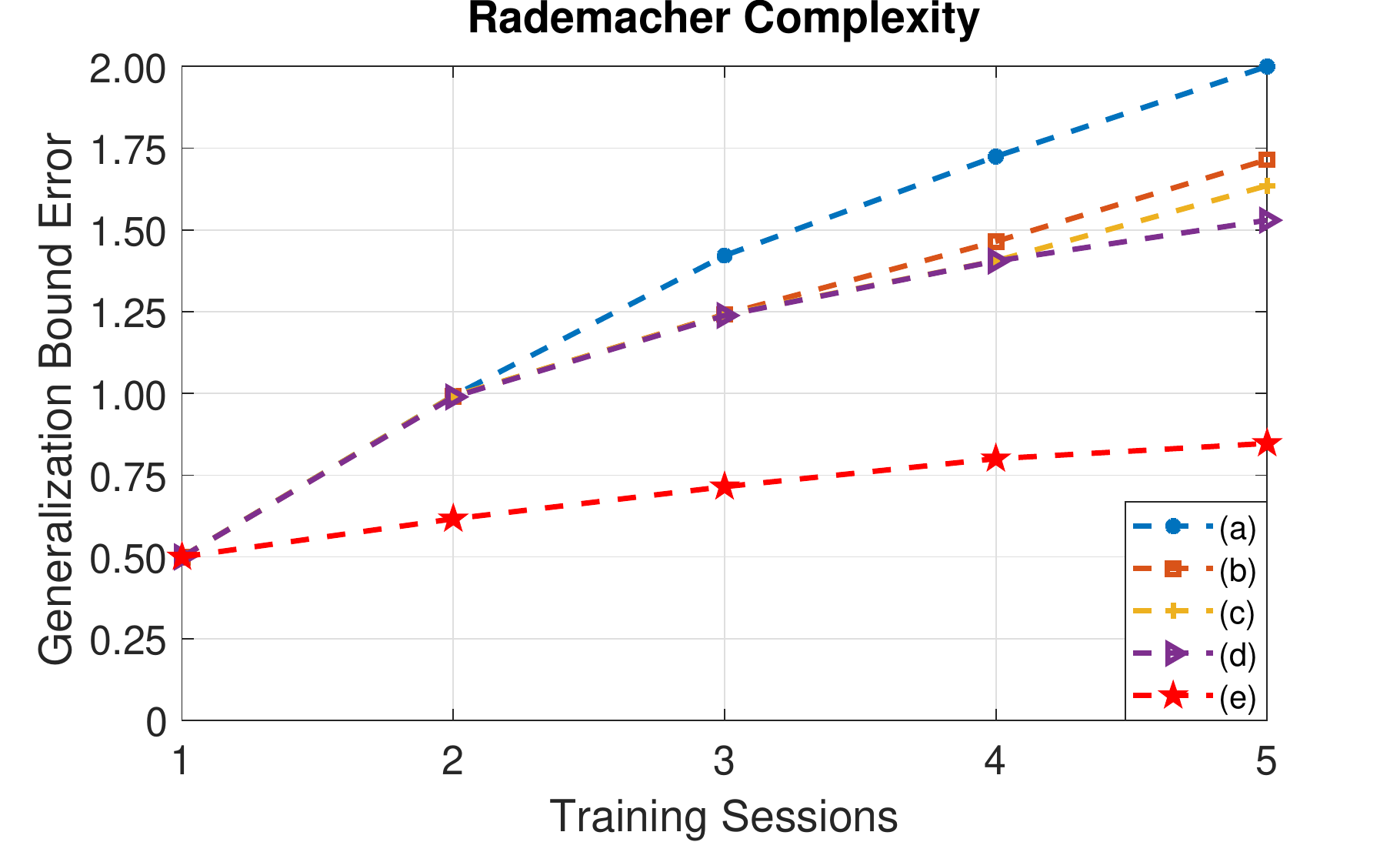}
	\caption{Empirical analysis of generalization bound error based on Rademacher complexity. Note that the legend used here is the same as that of the previous setting.}
	\label{Fig_Rademacher_complexity}
\end{figure}

%

\section{Conclusions}
Most continual learning approaches emphasize employing external (not inherent) conditions, such as exemplar buffers/data generators, additional objectives, and exclusive branches to alleviate forgetting, rather than focusing on an inherent implementation of avoiding overwritten weights within connectionist models. To this end, this study introduced two simple projections, randomization-based representation-level learning and orthogonalization-based decision-making, to innately overcome forgetting in a backbone network. The theoretical verification for convergence analysis of representation approximation and proof of the learning-without-forgetting decision were provided with rigorous mathematical deductions. The resulting method avoided forward and backward passes of backpropagation, and one could retain weights of the hidden layers invariable and exclusively adapt to that of the output layer during sequential training. Thus, the proposed method perfectly reproduced what was previously learned and relaxed the computational and memory requirements.  

Recent studies have highlighted the potential importance of leveraging a pre-trained model in the continual learning context. Based on this, we used a simple architecture to better analyze a network's behavior throughout this study, as challenging class-incremental learning can be performed with small adaptations. It is natural to consider whether IF2Net can be scaled to end-to-end implementation for challenging tasks. The purpose of this study was to provide a systematic solution with dual projections (randomization and orthogonalization) for continual learning challenges, and we leave it for future research.


%

%

\ifCLASSOPTIONcompsoc
\section*{Acknowledgments}
\else
\section*{Acknowledgment}
\fi
This work was supported by the National Key R\&D Program of China under Grant 2021ZD0201300, the Fundamental Research Funds for the Central Universities under Grant YCJJ202203012, the State Scholarship Fund of China Scholarship Council under Grant 202206160045, and the National Natural Science Foundation of China under Grants U1913602, 61936004, and 62206204.


\ifCLASSOPTIONcaptionsoff
\newpage
\fi



%
%
%
\bibliographystyle{IEEEtran}
\bibliography{IEEEabrv,mybibfile}

%


%




\end{document}